\documentclass[11pt]{article}
\usepackage{times}
\usepackage{fullpage}
\usepackage[colorlinks,
            linkcolor=blue,
            citecolor=blue,
            urlcolor=magenta,
            linktocpage,
            plainpages=false]{hyperref}
\usepackage{nameref}
\usepackage[numbers, compress]{natbib}
\usepackage[utf8]{inputenc} 
\usepackage[T1]{fontenc}    
\usepackage{url}            
\usepackage{booktabs}       
\usepackage{amsfonts}       
\usepackage{nicefrac}       
\usepackage{microtype}      

\usepackage{tabularx}

\usepackage{algorithm}
\usepackage{algorithmic}
\usepackage{amssymb}
\usepackage{amsmath}
\usepackage{amsthm}
\usepackage{enumitem}
\usepackage{verbatim}
\usepackage{xcolor}

\newtheorem{theorem}{Theorem}
\newtheorem{lemma}{Lemma}[theorem]
\newtheorem{claim}{Claim}[theorem]
\newtheorem*{corollary*}{Corollary}
\renewcommand{\iff}{\Leftrightarrow}
\newcommand{\OPT}{\texttt{OPT}}
\newcommand{\E}{\mathop{\mathbb{E}}}
\newcommand{\maj}{\textup{maj}}
\newcommand{\eps}{\epsilon}

\begin{document}

\title{Improved Algorithms for Collaborative PAC Learning}

\author{
   Huy L\^{e} Nguy\~{\^{e}}n \thanks{College of Computer and Information Science,
  Northeastern University.  \texttt{hu.nguyen@northeastern.edu}. This
work was supported by NSF CAREER 1750716.}
   \and
   Lydia Zakynthinou \thanks{College of Computer and Information Science (CCIS),
  Northeastern University.  \texttt{zakynthinou.l@northeastern.edu}. This
work was supported by a Graduate Fellowship from CCIS.}
}

\date{}

\maketitle

\begin{abstract}
  We study a recent model of collaborative PAC learning where $k$ players with $k$ different tasks collaborate to learn a single classifier that works for all tasks. Previous work showed that when there is a classifier that has very small error on all tasks, there is a collaborative algorithm that finds a single classifier for all tasks and has $O((\ln (k))^2)$ times the worst-case sample complexity for learning a single task. In this work, we design new algorithms for both the realizable and the non-realizable setting, having sample complexity only $O(\ln (k))$ times the worst-case sample complexity for learning a single task. The sample complexity upper bounds of our algorithms match previous lower bounds and in some range of parameters are even better than previous algorithms that are allowed to output different classifiers for different tasks.
\end{abstract}

\section{Introduction}
There has been a lot of work in machine learning concerning learning multiple tasks simultaneously, ranging from multi-task learning~\cite{Baxter97, Baxter00}, to domain adaptation~\cite{MMR09-1, MMR09-2}, to distributed learning~\cite{BBFM12, DGSX11, WKS16}. 
Another area in similar spirit to this work is meta-learning, where one leverages samples from many different tasks to train a single algorithm that adapts well to all tasks (see e.g.~\cite{FAL17}).

In this work, we focus on a model of collaborative PAC learning, proposed by~\cite{BHPQ17}. 
In the classic PAC learning setting introduced by~\cite{Valiant84}, where PAC stands for \textit{probably approximately correct}, the goal is to learn a task by drawing from a distribution of samples. 
The optimal classifier that achieves the lowest error on the task with respect to the given distribution is assumed to come from a concept class $\mathcal{F}$ of VC dimension $d$. 
The VC theorem~\cite{AB09} states that for any instance $m_{\eps,\delta} = O\Big(\frac{1}{\eps}\Big(d\ln \Big(\frac{1}{\eps}\Big) + \ln \Big(\frac{1}{\delta}\Big)\Big)\Big)$ labeled samples suffice to learn a classifier that achieves low error with probability at least $1-\delta$, where the error depends on $\eps$. 

In the collaborative model, there are $k$ players attempting to learn their own tasks, each task involving a different distribution of samples. 
The goal is to learn a single classifier that also performs well on all the tasks. One example from~\cite{BHPQ17}, which motivates this problem, is having $k$ hospitals with different patient demographics which want to predict the overall occurrence of a disease. In this case, it would be more fitting as well as cost efficient to develop and distribute a single classifier to all the hospitals. In addition, the requirement for a single classifier is imperative in settings where there are fairness concerns. For example, consider the case that the goal is to find a classifier that predicts loan defaults for a bank by gathering information from bank stores located in neighborhoods with diverse socioeconomic characteristics. In this setting, the samples provided by each bank store come from different distributions while it is desired to guarantee low error rates for all the neighborhoods. Again, in this setting, the bank should employ a single classifier among all the neighborhoods. 

If each player were to learn a classifier for their task without collaboration, they would each have to draw a sufficient number of samples from their distribution to train their classifier. Therefore, solving $k$ tasks independently would require $k \cdot m_{\eps,\delta}$ samples in the worst case. Thus, we are interested in algorithms that utilize samples from all players and solve all $k$ tasks with sample complexity $o\Big(\frac{k}{\eps}\Big(d\ln \Big(\frac{1}{\eps}\Big) + \ln \Big(\frac{1}{\delta}\Big)\Big)\Big)$.

Blum et al.~\cite{BHPQ17} give an algorithm with sample complexity $O\Big(\frac{\ln^2(k)}{\eps}\Big((d+k)\ln\Big(\frac{1}{\eps}\Big)+k\ln \Big(\frac{1}{\delta}\Big)\Big) \Big)$ for the realizable setting, that is, assuming the existence of a single classifier with zero error on all the tasks. 
They also extend this result by proving that a slightly modified algorithm returns a classifier with error $\eps$, under the relaxed assumption that there exists a classifier with error $\eps/100$ on all the tasks. 
In addition, they prove a lower bound showing that there is a concept class with $d=\Theta(k)$ where $\Omega\Big(\frac{k}{\eps}\ln\Big(\frac{k}{\delta}\Big)\Big)$ samples are necessary.

In this work, we give two new algorithms based on multiplicative weight updates which have sample complexities $O\Big(\frac{\ln (k)}{\eps}\Big(d\ln\Big(\frac{1}{\eps}\Big)+k\ln\Big(\frac{k}{\delta}\Big)\Big)\Big)$ and $O\Big(\frac{1}{\eps}\ln \Big(\frac{k}{\delta}\Big)\Big(d\ln\Big(\frac{1}{\eps}\Big)+k+\ln \Big(\frac{1}{\delta}\Big)\Big) \Big)$ for the realizable setting. 
Our first algorithm matches the sample complexity of~\cite{BHPQ17} for the variant of the problem in which the algorithm is allowed to return different classifiers to the players and our second algorithm has the sample complexity almost matching the lower bound of ~\cite{BHPQ17} when $d=\Theta(k)$ and for typical values of $\delta$.
Both are presented in Section~\ref{realizable}. 	Independently of our work, \cite{CZZ18} use the multiplicative weight update approach and achieve the same bounds as we do in that section.

Moreover, in Section~\ref{non_realizable}, we extend our results to the non-realizable setting, presenting two algorithms that generalize the algorithms for the realizable setting. These algorithms learn a classifier with error at most $(2+\alpha)\OPT+\eps$ on all the tasks, where $\alpha$ is set to a constant value, and have sample complexities $O\Big(\frac{\ln (k)}{\alpha^4\eps}\Big(d\ln\Big(\frac{1}{\eps}\Big)+k\ln\Big(\frac{k}{\delta}\Big)\Big)\Big)$ and $O\Big(\frac{1}{\alpha^4\eps}\ln \Big(\frac{k}{\delta}\Big)\Big(d\ln\Big(\frac{1}{\eps}\Big)+k\ln\Big(\frac{1}{\alpha}\Big)+\ln \Big(\frac{1}{\delta}\Big)\Big) \Big)$. With constant $\alpha$, these sample complexities are the same as in the realizable case. Finally, we give two algorithms with {\em randomized} classifiers whose error probability over the random choice of the example {\em and the classifier's randomness} is at most $(1+\alpha)\OPT+\eps$ for all tasks. The sample complexities of these algorithms are $O\Big(\frac{\ln (k)}{\alpha^3\eps^2}\Big(d\ln\Big(\frac{1}{\eps}\Big)+k\ln\Big(\frac{k}{\delta}\Big)\Big)\Big)$ and $O\Big(\frac{1}{\alpha^3\eps^2}\ln \Big(\frac{k}{\delta}\Big)\Big((d+k)\ln\Big(\frac{1}{\eps}\Big)+\ln \Big(\frac{1}{\delta}\Big)\Big) \Big)$.

\section{Model}
In the traditional PAC learning model, there is a space of instances $\mathcal{X}$ and a set $\mathcal{Y}=\{0,1\}$ of possible labels for the elements of $\mathcal{X}$. 
A classifier $f:\mathcal{X}\rightarrow\mathcal{Y}$, which matches each element of $\mathcal{X}$ to a label, is called a \textit{hypothesis}. 
The error of a hypothesis with respect to a distribution $D$ on $\mathcal{X}\times\mathcal{Y}$ is defined as $\text{err}_{D}(f)=\Pr_{(x,y) \sim D} [f(x)\neq y]$. 
Let $\OPT = \inf\limits_{ f \in \mathcal{F} } \text{err}_{D}(f)$, where $\mathcal{F}$ is a class of hypotheses. 
In the realizable setting we assume that there exists a target classifier with zero error, that is, there exists $f^*\in \mathcal{F}$ with $\text{err}_{D}(f^*)= \OPT = 0$ for all $i\in [k]$. Given parameters $(\eps,\delta)$, the goal is to learn a classifier that has error at most $\eps$, with probability at least $1-\delta$.
In the non-realizable setting, the optimal classifier $f^*$ is defined to have $\text{err}_{D}(f^*)\leq \OPT+\varepsilon$ for any $\varepsilon>0$. Given parameters $(\eps,\delta)$ and a new parameter $\alpha$, which can be considered to be a constant, the goal is to learn a classifier that has error at most $(1+\alpha)\OPT+\eps$, with probability at least $1-\delta$. 

By the VC theorem and its known extension, the desired guarantee can be achieved in both settings by drawing a set of samples of size $m_{\eps,\delta} = O\Big(\frac{1}{\eps}\Big(d\ln \Big(\frac{1}{\eps}\Big) + \ln \Big(\frac{1}{\delta}\Big)\Big)\Big)$ and returning the classifier with minimum error on that sample. 
More precisely, in the non-realizable setting, $m_{\eps,\delta} = \frac{C}{\eps\alpha}\Big(d\ln \Big(\frac{1}{\eps}\Big) + \ln \Big(\frac{1}{\delta}\Big)\Big)$, where $C$ is also a constant. 
We consider an algorithm $\mathcal{O}_{\mathcal{F}}(S)$, where $S$ is a set of samples drawn from an arbitrary distribution $D$ over the domain $\mathcal{X} \times \{0, 1\}$, that returns a hypothesis $f_0$ whose error on the sample set satisfies $\text{err}_{S}(f_0) \leq \inf\limits_{f\in \mathcal{F}} \text{err}_{S}(f) + \varepsilon$ for any $\varepsilon>0$, if such a hypothesis exists. The VC theorem guarantees that if $|S|=m_{\eps,\delta}$, then $\text{err}_{D}(f_0) \leq (1+\alpha)\text{err}_{S}(f_0) + \eps$.

In the collaborative model, there are $k$ players with distributions $D_1,\ldots, D_k$. 
Similarly, $\OPT = \inf\limits_{ f \in \mathcal{F} }  \max\limits_{i \in [k]} \text{err}_{D_i}(f)$ and the goal is to learn a single good classifier for all distributions. 
In~\cite{BHPQ17}, the authors consider two variants of the model for the realizable setting, the \textit{personalized} and the \textit{centralized}. 
In the former the algorithm can return a different classifier to each player, while in the latter it must return a single good classifier.
For the personalized variant, Blum et al. give an algorithm with almost the same sample complexity as the lower bound they provide. 
We focus on the more restrictive centralized variant of the model, for which the algorithm that Blum et al. give does not match the lower bound.
We note that the algorithms we present are \textit{improper}, meaning that the classifier they return is not necessarily in the concept class $\mathcal{F}$.

\section{Sample complexity upper bounds for the realizable setting}\label{realizable}

In this section, we present two algorithms and prove their sample complexity. 

Both algorithms employ multiplicative weight updates, meaning that in each round they find a classifier with low error on the weighted mixture of the distributions and double the weights of the players for whom the classifier did not perform well. In this way, the next sample set drawn will include more samples from these players' distributions so that the next classifier will perform better on them. To identify the players for whom the classifier of the round did not perform well, the algorithms test the classifier on a small number of samples drawn from each player's distribution. If the error of the classifier on the sample is low, then the error on the player's distribution can not be too high and vise versa. In the end, both algorithms return the majority function over all the classifiers of the rounds, that is, for each point $x\in\mathcal{X}$, the label assigned to $x$ is the label that the majority of the classifiers assign to $x$. 

We note that for typical values of $\delta$, Algorithm R2 is better than Algorithm R1. However, Algorithm R1 is always better than the algorithm of~\cite{BHPQ17} for the centralized variant of the problem and matches their number of samples in the personalized variant, so we present both algorithms in this section.
In the algorithms of~\cite{BHPQ17}, the players are divided into classes based on the number of rounds for which that player's task is not solved with low error. The number of classes could be as large as the number of rounds, which is $\Theta(\log (k))$, and their algorithm uses roughly $m_{\eps, \delta}$ samples from each class. On the other hand, Algorithm R1 uses only $m_{\eps, \delta}$ samples across all classes and saves a factor of $\Theta(\log (k))$ in the sample complexity. This requires analyzing the change in all classes together as opposed to class by class.

\begin{algorithm}[H]
\renewcommand{\thealgorithm}{}
\caption{R1}\label{algR1}
\begin{algorithmic}
\STATE {\bfseries Initialize:} $\forall i\in [k] $ $w_i^{(0)}:= 1$; $t:=5\lceil \log (k) \rceil$; $\eps':=\eps/6$; $\delta':=\delta/(3t)$;
\FOR{$r = 1$ {\bfseries to} $t$}
\STATE $\tilde{D}^{(r-1)} \gets \frac{1}{\Phi^{(r-1)}}\sum_{i=1}^k \left(w_i^{(r-1)}D_i\right)$, where $\Phi^{(r-1)} = \sum_{i=1}^k w_i^{(r-1)}$;
\STATE Draw a sample set $S^{(r)}$ of size $m_{\eps'/16, \delta'}$ from $\tilde{D}^{(r-1)}$;
\STATE $f^{(r)} \gets \mathcal{O}_{\mathcal{F}}(S^{(r)})$;
\STATE $G_r \gets \textsc{Test}(f^{(r)}, k, \eps', \delta')$;
\STATE \textbf{Update}: $w_i^{(r)} = \begin{cases} 2w_i^{(r-1)}, & \mbox{if } i\notin G_r \\ 
w_i^{(r-1)}, & \mbox{otherwise} \end{cases}$ ;
\ENDFOR
\STATE {\bfseries return} $f_{\text{R1}} = \maj(\{f^{(r)}\}_{r=1}^t)$
\STATE
\STATE {\bfseries Procedure} $\textsc{Test}(f^{(r)}, k, \eps', \delta')$
\FOR{$i=1$ {\bfseries to} $k$}
\STATE Draw a sample set $T_i$ of size $O\Big(\frac{1}{\eps'} \ln \Big( \frac{k}{\delta'}\Big)\Big)$ from $D_i$;
\ENDFOR
\STATE {\bfseries return} $\{i \mid \text{err}_{T_i}(f^{(r)})\leq \frac{3}{4}\eps'\}$;
\end{algorithmic}
\end{algorithm}

Algorithm R1 runs for $t=\Theta(\log (k))$ rounds and learns a classifier $f^{(r)}$ in each round $r$ that has low error on the weighted mixture of the distributions $\tilde{D}^{(r-1)}$. For each player at least $0.6t$ of the learned classifiers are ``good'', meaning that they have error at most $\eps'=\eps/6$ on the player's distribution. Since the algorithm returns the majority of the classifiers, in order for an instance to be mislabeled, at least $0.5t$ of the total number of classifiers should mislabel it. This implies that at least $0.1t$ of the ``good'' classifiers of that player should mislabel it, which amounts to $1/6$ of the ``good'' classifiers. Therefore, the error of the majority of the functions for that player is at most $6\eps'=\eps$. 

To identify the players for whom the classifier of the round does not perform well, Algorithm R1 uses a procedure called \textsc{Test}. This procedure draws $O\Big(\frac{1}{\eps'} \ln \Big( \frac{k}{\delta'}\Big)\Big)$ samples from each player's distribution and tests the classifier on these samples. If the error for a player's sample set is at most $3\eps'/4$ then \textsc{Test} concludes that the classifier is good for that player and adds them to the returned set $G_r$. The samples that the \textsc{Test} requires from each player suffice to make it capable of distinguishing between the players with error more than $\eps'$ and players with error at most $\eps'/2$ with respect to their distributions, with high probability.

\begin{theorem}\label{thR1}
For any $\eps, \delta \in (0,1)$, and hypothesis class $\mathcal{F}$ of VC dimension $d$, Algorithm R1 returns a classifier $f_{\text{R1}}$ with $\text{err}_{D_i}(f_{\text{R1}}) \leq \eps$ $\forall i\in[k]$ with probability at least $1-\delta$ using $m$ samples, where 
\[ m = O\Big(\frac{\ln (k)}{\eps}\Big(d\ln\Big(\frac{1}{\eps}\Big)+k\ln\Big(\frac{k}{\delta}\Big)\Big)\Big).\]
\end{theorem}

To prove the correctness and sample complexity of Algorithm R1, we need to prove Lemma~\ref{lemma:test}, which describes the set $G_r$ that the \textsc{Test} returns. This proof uses the following multiplicative forms of the Chernoff bounds (proved as in Theorems 4.4 and 4.5 of~\cite{MU17}). 

\begin{lemma}[Chernoff Bounds]\label{lemma:chernoff}
If $X$ is the average of $n$ independent random variables taking values in $\{0, 1\}$, then
\begin{equation}\label{chernoff1} 
\Pr[X \leq (1-s)\E[X]] \leq \exp\Big(-\frac{s^2 \E[X] n}{2}\Big),
\end{equation}
\begin{equation}\label{chernoff2}
\Pr[X \geq (1+s)\E[X]] \leq \exp\Big(-\frac{s^2 \E[X] n}{3}\Big),
\end{equation}
\begin{equation}\label{chernoff3}
\Pr[X \geq (1+s)\E[X]] \leq \exp\Big(-\frac{s \E[X] n}{3}\Big),
\end{equation}
where the latter inequality holds for $s\geq 1$ and the first two hold for $s\in(0,1)$.
\end{lemma}

\begin{lemma}\label{lemma:test}
$\textsc{Test}(f^{(r)}, k, \eps', \delta')$ is such that the following two properties hold, each with probability at least $1-\delta'$, for all $i\in [k]$ and for a given round $r\in [t]$.
\begin{enumerate}[label=(\alph*)]
\item If $\text{err}_{D_i}(f^{(r)}) > \eps'$, then $i\notin G_r$.
\item If $\text{err}_{D_i}(f^{(r)}) \leq\frac{\eps'}{2}$, then $i\in G_r$.
\end{enumerate}
\end{lemma}
\begin{proof}[Proof of Lemma~\ref{lemma:test}]
For this proof we assume that the number of samples $\vert T_i \vert $ for each $i\in [k]$ must be at least 
$\frac{32}{\eps'} \ln \Big(\frac{k}{\delta'}\Big) = O\Big(\frac{1}{\eps'} \ln \Big(\frac{k}{\delta'}\Big)\Big)$.
For a given round $r\in [t]$:
\begin{enumerate}[label=($\alph*$)]
\item Assume $\text{err}_{D_i}(f^{(r)}) > \eps'$ for some $i\in [k]$. Then

$\Pr\Big[i\in G_r\Big] \\
= \Pr \Big[\text{err}_{T_i}(f^{(r)}) \leq \frac{3}{4}\eps'\Big]\\
< \Pr\Big[\text{err}_{T_i}(f^{(r)}) \leq \Big(1-\frac{1}{4}\Big)\text{err}_{D_i}(f^{(r)})\Big]\\
\stackrel{\text{(\ref{chernoff1})}}{\leq} \exp\Big(-\frac{1}{2}\Big(\frac{1}{4}\Big)^2\text{err}_{D_i}(f^{(r)}) \vert T_i \vert\Big)\\
< \exp\Big(-\frac{1}{32}\eps'\vert T_i \vert\Big) \\
\leq \exp\Big(-\frac{1}{32}\eps'\frac{32}{\eps'} \ln \Big(\frac{k}{\delta'}\Big)\Big)\\
\leq \frac{\delta'}{k}.$

Hence, by union bound, $\text{err}_{D_i}(f^{(r)}) > \eps' \Rightarrow i\notin G_r$ holds for all $i\in [k]$ with probability at least $1-\delta'$.

\item Assume $\text{err}_{D_i}(f^{(r)}) \leq \frac{\eps'}{2}$ for some $i\in [k]$. We consider two cases and we apply the Chernoff bounds with $s = \frac{\eps'}{4\text{err}_{D_i}(f^{(r)})}$.
Note that if $\text{err}_{D_i}(f^{(r)})=0$ then $\text{err}_{T_i}(f^{(r)})=0$ and the property holds. So we only need to consider $\text{err}_{D_i}(f^{(r)})\neq0$.
First, we need to prove that

$\frac{3\eps'}{4} \geq (1+ s)\text{err}_{D_i}(f^{(r)})\\
\iff \frac{3\eps'}{4\text{err}_{D_i}(f^{(r)})} \geq 1+ \frac{\eps'}{4\text{err}_{D_i}(f^{(r)})}\\
\iff \frac{\eps'}{2\text{err}_{D_i}(f^{(r)})} \geq 1,$

which is true.

\begin{enumerate}[label=\textit{Case \arabic*.}]
\item If $\text{err}_{D_i}(f^{(r)}) > \frac{\eps'}{4}$, which implies $s<1$, then

$\Pr\Big[i\notin G_r\Big] \\
= \Pr \Big[\text{err}_{T_i}(f^{(r)}) > \frac{3}{4}\eps'\Big]\\
\leq \Pr\Big[\text{err}_{T_i}(f^{(r)})\geq \Big(1+s\Big)\text{err}_{D_i}(f^{(r)})\Big]\\
 \stackrel{\text{(\ref{chernoff2})}}{\leq} \exp\Big(-\frac{1}{3}\Big(\frac{\eps'}{4\text{err}_{D_i}(f^{(r)})}\Big)^2\text{err}_{D_i}(f^{(r)}) \vert T_i \vert\Big)\\
 = \exp\Big(-\frac{\eps'^2}{48\text{err}_{D_i}(f^{(r)})}\vert T_i \vert\Big)\\
 \leq \exp\Big(-\frac{1}{48}2\eps' \frac{24}{\eps'}\ln \Big(\frac{k}{\delta'}\Big)\Big)\\
 \leq \frac{\delta'}{k}.$

\item If $\text{err}_{D_i}(f^{(r)}) \leq \frac{\eps'}{4}$, which implies $s\geq 1$, then:

$\Pr\Big[i\notin G_r\Big] \\
= \Pr \Big[\text{err}_{T_i}(f^{(r)}) > \frac{3}{4}\eps'\Big]\\
 \leq \Pr\Big[\text{err}_{T_i}(f^{(r)}) \geq \Big(1+s\Big)\text{err}_{D_i}(f^{(r)})\Big]\\
 \stackrel{\text{(\ref{chernoff3})}}{\leq} \exp\Big(-\frac{1}{3}\frac{\eps'}{4\text{err}_{D_i}(f^{(r)})}\text{err}_{D_i}(f^{(r)}) \vert T_i \vert\Big)\\
 = \exp\Big(-\frac{\eps'}{3}\vert T_i \vert\Big)\\
  \leq \exp\Big(-\frac{\eps'}{3} \frac{3}{\eps'}\ln \Big(\frac{k}{\delta'}\Big)\Big)\\
 \leq \frac{\delta'}{k}.$
\end{enumerate}
Hence, by union bound, $\text{err}_{D_i}(f^{(r)}) \leq \frac{\eps'}{2} \Rightarrow i\in G_r$ holds for all $i\in [k]$ with probability at least $1-\delta'$.
\end{enumerate}
\end{proof}

Having proven Lemma~\ref{lemma:test}, we can now prove Theorem~\ref{thR1}.

\begin{proof}[Proof of Theorem 1]
First, we prove that Algorithm R1 indeed learns a good classifier, meaning that for every player $i\in [k]$ the returned classifier $f_{\text{R1}}$ has error $\text{err}_{D_i}(f_{\text{R1}})\leq \eps$ with probability at least $1-\delta$.

Let $e_i^{(r)}$ denote the number of rounds, up until and including round $r$, that $i$ did not pass the \textsc{Test}. More formally, $e_i^{(r)}= \vert \{r' \mid r' \in [r] \text{ and } i\notin G_{r'}\}\vert$.

\begin{claim}[]\label{claim1}
With probability at least $1-\frac{2\delta}{3}$, $e_i^{(t)}<0.4t$ $\forall i\in[k]$.
\end{claim}

From Lemma~\ref{lemma:test}($a$) and union bound, with probability at least $1-t\delta'=1-\frac{\delta}{3}$, the number of functions that have error more than $\eps'$ on $D_i$ is the same as the number of rounds that $i$ did not pass the \textsc{Test}, for all $i\in [k]$. So, if the claim holds, with probability at least $1-(\frac{2}{3}+\frac{1}{3})\delta=1-\delta$, less than $0.4t$ functions have error more than $\eps'$ on $D_i$, for all $i\in [k]$. Equivalently, with probability at least $1-\delta$, more than $0.6t$ functions have error at most $\eps'$ on $D_i$, for all $i\in [k]$. As a result, with probability at least $1-\delta$, the error of the majority of the functions is $\text{err}_{D_i}(f_{\text{R1}})  \leq \frac{0.6}{0.1} \eps' = \eps$ for all $i\in[k]$.

Let us now prove the claim.

\begin{proof}[Proof of Claim~\ref{claim1}]
\renewcommand\qedsymbol{$\blacksquare$}
Recall that $\Phi^{(r)}=\sum_{i=1}^k w_i^{(r)}$ is the potential function in round $r$. By linearity of expectation, the following holds for the error on the mixture of distributions:
\begin{equation}\label{doubledphi}
\begin{array}{ll@{}ll}
\text{err}_{\tilde{D}^{(r-1)}}(f^{(r)}) 
 &=\frac{1}{\Phi^{(r-1)}}\sum_{i=1}^k \left(w_i^{(r-1)}\text{err}_{D_i}(f^{(r)})\right)\\
 &\geq \frac{1}{\Phi^{(r-1)}} \sum_{i\notin G_r} \left(w_i^{(r-1)}\text{err}_{D_i}(f^{(r)})\right)
\end{array}
\end{equation}

From the VC theorem, it holds that, since $f^{(r)}=\mathcal{O}_{\mathcal{F}}(S^{(r)})$ and $\vert S^{(r)} \vert = m_{\eps'/16,\delta'}$, with probability at least $1-\delta'$, $\text{err}_{\tilde{D}^{(r-1)}}(f^{(r)}) \leq \frac{\eps'}{16}$. From Lemma~\ref{lemma:test}($b$), with probability at least $1-\delta'$, $\text{err}_{D_i}(f^{(r)})\geq \frac{\eps'}{2}$ for all $i\notin G_r$.
So with probability at least $1-2\delta'$ the two hold simultaneously. Combining these inequalities with (\ref{doubledphi}), we get that with probability at least $1-2\delta'$,
$\frac{\epsilon'}{16} \geq \frac{1}{\Phi^{(r-1)}}\sum_{i=1}^k \left(w_i^{(r-1)} \frac{\epsilon'}{2}\right) \iff \sum_{i\notin G_r} w_i^{(r-1)}\leq \frac{1}{8} \Phi^{(r-1)}.$

Since only the weights of players $i\notin G_r$ are doubled, it holds that for a given round $r$ 
$$\Phi^{(r)} \leq \Phi^{(r-1)} + \sum_{i\notin G_r} w_i^{(r-1)} \leq \frac{9}{8}\Phi^{(r-1)}.$$
Therefore with probability at least $1-2t\delta'=1-\frac{2\delta}{3}$, the inequality holds for all rounds, by union bound. By induction: 
$$\Phi^{(t)}\leq \Big(\frac{9}{8}\Big)^t \Phi^{(0)}=\Big(\frac{9}{8}\Big)^t k$$

Also, for every $i\in [k]$ it holds that $w_i^{(t)}=2^{e_i^{(t)}}$, as each weight is only doubled every time $i$ does not pass the \textsc{Test}. Since the potential function is the sum of all weights, the following inequality is true.

$w_i^{(t)}\leq \Phi^{(t)}\\
\Rightarrow 2^{e_i^{(t)}} \leq \Big( \frac{9}{8} \Big)^t k\\
\Rightarrow e_i^{(t)} \leq t\log\Big(\frac{9}{8}\Big) + \log (k)\\
\Rightarrow e_i^{(t)} \leq 0.17t+0.2t<0.4t$

So with probability at least $1-\frac{2\delta}{3}$, $e_i^{(t)} < 0.4t$ $\forall i\in[k]$.
\end{proof}

As for the total number of samples, it is the sum  of \textsc{Test}'s samples and the $m_{\eps'/16,\delta'}$ samples for each round.
Since \textsc{Test} is called $t=5\lceil \log (k) \rceil$ times and each time requests $O\Big(\frac{1}{\eps'} \ln \Big( \frac{k}{\delta'}\Big)\Big)$ samples from each of the $k$ players, the total number of samples that it requests is $O\Big(\log (k) \frac{k}{\eps'} \ln \Big( \frac{k}{\delta'}\Big)\Big)$. Substituting $\eps'=\eps/6$ and $\delta' = \delta/(3t) = \delta/(15 \lceil \log (k) \rceil)$, this yields
$$O\Big(\frac{\log (k)}{\eps} k\ln \Big(\frac{k\log (k)}{\delta}\Big)\Big)
= O\Big(\frac{\log (k)}{\eps} k\ln \Big(\frac{k}{\delta}\Big)\Big)$$
samples in total.

 In addition, the sum of the $m_{\eps'/16,\delta'}$ samples drawn in each round to learn the classifier for the mixture for $t=5\lceil \log (k) \rceil$ rounds is
$O\Big(\frac{\log (k)}{\eps'} \Big(d\ln \Big(\frac{1}{\eps'}\Big)+\ln \Big(\frac{1}{\delta'}\Big)\Big)\Big)$. Again, substituting $\eps'$ and $\delta'$, we get:
$$O\Big(\frac{\log (k)}{\eps} \Big(d\ln \Big(\frac{1}{\eps}\Big)+\ln \Big(\frac{\log (k)}{\delta}\Big)\Big)\Big)$$
samples in total.

Hence, the overall bound is:
$$O\Big(\frac{\log (k)}{\eps} \Big(d\ln \Big(\frac{1}{\eps}\Big)+k\ln \Big(\frac{k}{\delta}\Big)\Big)\Big)$$
\end{proof}

Algorithm R1 is the natural boosting alternative to the algorithm of~\cite{BHPQ17} for the centralized variant of the model. Although it is discussed in~\cite{BHPQ17} and mentioned to have the same sample complexity as their algorithm, it turns out that it is more efficient. Its sample complexity is slightly better (or the same, depending on the parameter regime) compared to the one of the algorithm for the personalized setting presented in~\cite{BHPQ17}, which is $O\Big(\frac{\log (k)}{\eps} \Big((d+k)\ln \Big(\frac{1}{\eps}\Big)+k\ln \Big(\frac{k}{\delta}\Big)\Big)\Big)$. 

However, in the setting of the lower bound in~\cite{BHPQ17} where $k=\Theta(d)$, there is a gap of $\log (k)$ multiplicatively between the sample complexity of Algorithm R1 and the lower bound. 
This difference stems from the fact that in every round, the algorithm uses roughly $\Theta(k)$ samples to find a classifier but roughly $\Theta( k\log (k))$ samples to test the classifier for $k$ tasks. 
Motivated by this discrepancy, we develop Algorithm R2, which is similar to Algorithm R1 but uses fewer samples to test the performance of each classifier on the players' distributions. 
To achieve high success probability, Algorithm R2 uses a higher number of rounds.
 
\begin{algorithm}[H]
\renewcommand{\thealgorithm}{}
\caption{R2}\label{algR2}
\begin{algorithmic}
\STATE {\bfseries Initialize:} $\forall i\in [k] $ $w_i^{(0)}:=1$; $t:=150\Big\lceil \log \Big(\frac{k}{\delta} \Big) \Big\rceil$; $\eps':=\eps/6$; $\delta':=\delta/(4t)$;
\FOR{$r = 1$ {\bfseries to} $t$}
\STATE $\tilde{D}^{(r-1)} \gets \frac{1}{\Phi^{(r-1)}}\sum_{i=1}^k \left(w_i^{(r-1)}D_i\right)$, where $\Phi^{(r-1)} = \sum_{i=1}^k w_i^{(r-1)}$;
\STATE Draw a sample set $S^{(r)}$ of size $m_{\eps'/16, \delta'}$ from $\tilde{D}^{(r-1)}$;
\STATE $f^{(r)} \gets \mathcal{O}_{\mathcal{F}}(S^{(r)})$;
\STATE $G_r \gets \textsc{FastTest}(f^{(r)}, k, \eps', \delta')$;
\STATE \textbf{Update}: $w_i^{(r)} = \begin{cases} 2w_i^{(r-1)}, & \mbox{if } i\notin G_r \\ 
w_i^{(r-1)}, & \mbox{otherwise} \end{cases}$ ;
\ENDFOR
\STATE {\bfseries return} $f_{\text{R2}} = \maj ( \{ f^{(r)} \}_{r=1}^t )$;
\STATE
\STATE {\bfseries Procedure} $\textsc{FastTest}(f^{(r)}, k, \eps', \delta')$
\FOR{$i=1$ {\bfseries to} $k$}
\STATE Draw a sample set $T_i$ of size $O\Big(\frac{1}{\eps'}\Big)$ from $D_i$;
\ENDFOR
\STATE {\bfseries return} $\{i \mid \text{err}_{T_i}(f^{(r)})\leq \frac{3}{4}\eps'\}$;
\end{algorithmic}
\end{algorithm}

More specifically, Algorithm R2 runs for $t=150\lceil \log(\frac{k}{\delta})\rceil$ rounds. 
In addition, the test it uses to identify the players for whom the classifier of the round does not perform well requires $O\Big(\frac{1}{\eps'}\Big)$ samples from each player. 
This helps us save one logarithmic factor in the second term of the sample complexity of Algorithm R1. 
We call this new test \textsc{FastTest}. 
The fact that \textsc{FastTest} uses less samples causes it to be less successful at distinguishing the players for whom the classifier was ``good'' from the players for whom it was not, meaning that it has constant probability of making a mistake for a given player at a given round. 
There are two types of mistakes that \textsc{FastTest} can make: to return $i\notin G_r$ and double the weight of $i$ when the classifier is good for $i$'s distribution and to return $i\in G_r$ and not double the weight of $i$ when the classifier is not good. 

\begin{theorem}\label{thR2}
For any $\eps, \delta \in (0,1)$, and hypothesis class $\mathcal{F}$ of VC dimension $d$, Algorithm R2 returns a classifier $f_{\text{R2}}$ with $\text{err}_{D_i}(f_{\text{R2}}) \leq \eps$ $\forall i\in[k]$ with probability at least $1-\delta$ using $m$ samples, where 
\[ m = O\Big(\frac{1}{\eps}\ln \Big(\frac{k}{\delta}\Big)\Big(d\ln\Big(\frac{1}{\eps}\Big)+k+\ln \Big(\frac{1}{\delta}\Big)\Big) \Big).\]
\end{theorem}

To prove the correctness and sample complexity of Algorithm R2, we need Lemma~\ref{lemma:fasttest}, which describes the set $G_r$ that the \textsc{FastTest} returns and is proven similarly to Lemma~\ref{lemma:test}.

\begin{lemma}\label{lemma:fasttest}
$\textsc{FastTest}(f^{(r)}, k, \eps', \delta')$ is such that the following two properties hold, each with probability at least $0.99$, for given round $r\in [t]$ and player $i\in [k]$.
\begin{enumerate}[label=(\alph*)]
\item If $\text{err}_{D_i}(f^{(r)}) > \eps'$, then $i\notin G_r$.
\item If $\text{err}_{D_i}(f^{(r)}) \leq\frac{\eps'}{2}$, then $i\in G_r$.
\end{enumerate}
\end{lemma}

\begin{proof}[Proof of Lemma 2.1]
For this proof, we assume that the number of samples $\vert T_i \vert $ for each $i\in [k]$ must be at least 
$\frac{148}{\eps'} = O\Big(\frac{1}{\eps'}\Big)$.
For given $r\in [t]$ and $i\in [k]$:
\begin{enumerate}[label=(\alph*)]
\item Assume $\text{err}_{D_i}(f^{(r)}) > \eps'$. Then

$\Pr\Big[i\in G_r\Big]\\
 = \Pr \Big[\text{err}_{T_i}(f^{(r)}) \leq \frac{3}{4}\eps'\Big]\\
 < \Pr\Big[\text{err}_{T_i}(f^{(r)}) \leq \Big(1-\frac{1}{4}\Big)\text{err}_{D_i}(f^{(r)})\Big]\\
 \stackrel{\text{(\ref{chernoff1})}}{\leq} \exp\Big(-\frac{1}{2}\Big(\frac{1}{4}\Big)^2\text{err}_{D_i}(f^{(r)}) \vert T_i \vert\Big)\\
 < \exp\Big(-\frac{1}{32}\eps'\vert T_i \vert\Big)\\
 \leq \exp\Big(-\frac{1}{32}\eps'\frac{148}{\eps'}\Big)\\
 < 0.01.$

Hence, $\text{err}_{D_i}(f^{(r)}) > \eps' \Rightarrow i\notin G_r$ holds with probability at least $0.99$.

\item Assume $\text{err}_{D_i}(f^{(r)}) \leq \frac{\eps'}{2}$. We consider two cases and we apply the Chernoff bounds with $s = \frac{\eps'}{4\text{err}_{D_i}(f^{(r)})}$.
Note that if $\text{err}_{D_i}(f^{(r)})=0$ then $\text{err}_{T_i}(f^{(r)})=0$ and the property holds. So we only need to consider $\text{err}_{D_i}(f^{(r)})\neq0$.
First, we need to prove that

$\frac{3\eps'}{4} \geq (1+ s)\text{err}_{D_i}(f^{(r)})\\
\iff \frac{3\eps'}{4\text{err}_{D_i}(f^{(r)})} \geq 1+ \frac{\eps'}{4\text{err}_{D_i}(f^{(r)})}\\
\iff \frac{\eps'}{2\text{err}_{D_i}(f^{(r)})} \geq 1,$

which is true.

\begin{enumerate}[label=\textit{Case \arabic*.}]
\item If $\text{err}_{D_i}(f^{(r)}) > \frac{\eps'}{4}$, which implies $s<1$, then

$\Pr\Big[i\notin G_r\Big]\\
 = \Pr \Big[\text{err}_{T_i}(f^{(r)}) > \frac{3}{4}\eps'\Big]\\
\leq \Pr\Big[\text{err}_{T_i}(f^{(r)}) \geq \Big(1+s\Big)\text{err}_{D_i}(f^{(r)})\Big]\\
 \stackrel{\text{(\ref{chernoff2})}}{\leq} \exp\Big(-\frac{1}{3}\Big(\frac{\eps'}{4\text{err}_{D_i}(f^{(r)})}\Big)^2\text{err}_{D_i}(f^{(r)}) \vert T_i \vert\Big)\\
 = \exp\Big(-\frac{\eps'^2}{48\text{err}_{D_i}(f^{(r)})}\vert T_i \vert\Big)\\
 \leq \exp\Big(-\frac{1}{48}2\eps' \frac{148}{\eps'}\Big)\\
 < 0.01.$
 
\item If $\text{err}_{D_i}(f^{(r)}) \leq \frac{\eps'}{4}$, which implies $s\geq 1$, then

$\Pr\Big[i\notin G_r\Big]\\
 = \Pr \Big[\text{err}_{T_i}(f^{(r)}) > \frac{3}{4}\eps'\Big]\\
\leq \Pr\Big[\text{err}_{T_i}(f^{(r)}) \geq \Big(1+s\Big)\text{err}_{D_i}(f^{(r)})\Big]\\
 \stackrel{\text{(\ref{chernoff3})}}{\leq} \exp\Big(-\frac{1}{3}\frac{\eps'}{4\text{err}_{D_i}(f^{(r)})}\text{err}_{D_i}(f^{(r)}) \vert T_i \vert\Big)\\
 = \exp\Big(-\frac{\eps'}{12}\vert T_i \vert\Big)\\
 \leq \exp\Big(-\frac{\eps'}{12} \frac{148}{\eps'}\Big)\\
 < 0.01.$
\end{enumerate}
Hence, $\text{err}_{D_i}(f^{(r)}) \leq \frac{\eps'}{2} \Rightarrow i\in G_r$ holds with probability at least $0.99$.
\end{enumerate}
\end{proof}

\begin{proof}[Proof of Theorem~\ref{thR2}]
First, we prove that Algorithm R2 indeed learns a good classifier, meaning that, with probability at least $1-\delta$, for every player $i\in [k]$ the returned classifier $f_{\text{R2}}$ has error $\text{err}_{D_i}(f_{\text{R2}})\leq \eps$. Let $e_i^{(t)}$ be the number of rounds for which the classifier's error on $D_i$ was more than $\eps'$, i.e. $e_i^{(t)} = \vert \{r \mid r \in [t] \text{ and } \text{err}_{D_i}(f^{(r)})>\eps'\}\vert$.

\begin{claim}[]\label{claim2}
With probability at least $1-\delta$, $e_i^{(t)}<0.4t$ $\forall i\in[k]$.
\end{claim}

If the claim holds, then with probability at least $1-\delta$, less than $0.4t$ functions have error more than $\eps'$ on $D_i$, $\forall i\in [k]$. Therefore, with probability at least $1-\delta$, $\text{err}_{D_i}(f_{\text{R2}}) \leq \frac{0.6}{0.1}\eps' \leq \eps$ for every $i\in [k]$.

\begin{proof}[Proof of Claim~\ref{claim2}]
\renewcommand\qedsymbol{$\blacksquare$}
Let us denote by $I^{(r)}$ the set of players having  $\text{err}_{D_i}(f^{(r)}) > \frac{\eps'}{2}$ in round $r$, i.e., $I^{(r)}= \{i \in [k] \mid \text{err}_{D_i}(f^{(r)}) > \frac{\eps'}{2}\}$. 
We condition on the randomness in the first $r-1$ rounds and compute $\E[\Phi^{(r)} \mid \Phi^{(r-1)}]$. By linearity of expectation, the following hold for round $r$:
\begin{equation}\label{doubledq}
\text{err}_{\tilde{D}^{(r-1)}}(f^{(r)}) 
 =\frac{1}{\Phi^{(r-1)}}\sum\limits_{i=1}^k \left(w_i^{(r-1)}\text{err}_{D_i}(f^{(r)})\right)
 \geq \frac{1}{\Phi^{(r-1)}} \sum\limits_{i\in I^{(r)} \setminus G_r} \left(w_i^{(r-1)}\text{err}_{D_i}(f^{(r)}) \right)
\end{equation}
By the definition of $I^{(r)}$, $\text{err}_{D_i}(f^{(r)}) > \frac{\eps'}{2}$ for $i\in I^{(r)}$. 
From the VC theorem, with probability at least $1-\delta'$, $\text{err}_{\tilde{D}^{(r-1)}}(f^{(r)}) \leq \frac{\eps'}{16}$. 
Using these two bounds and inequality~(\ref{doubledq}), it follows that with probability at least $1-\delta'$,
\begin{equation}\label{qbound1}
 \sum_{i\in I^{(r)}\setminus G_r} w_i^{(r-1)} \leq \frac{1}{8}\Phi^{(r-1)}.
\end{equation}

For the rest of the analysis, we will condition our probability space to the event that inequality~(\ref{qbound1}) holds for all $t$ rounds. By the union bound, this event happens with probability $1-t\delta' = 1-\delta/4$.

Consider the set of players $i\notin I^{(r)}\cup G_r$. These are the players for whom the classifier of the round performed well but \textsc{FastTest} made a mistake and did not include them in the set $G_r$. By linearity of expectation:
\begin{equation}\label{qbound2}
\begin{array}{ll@{}ll}
\E[\sum\limits_{i\notin G_r}{w_i^{(r-1)}}  \mid \Phi^{(r-1)}] &
=\E\left[\sum\limits_{i\in I^{(r)}\setminus G_r} w_i^{(r-1)} +  \sum\limits_{i\notin I^{(r)}\cup G_r} w_i^{(r-1)} \middle| \Phi^{(r-1)}\right]\\
& \stackrel{\text{(\ref{qbound1}), Lemma~\ref{lemma:fasttest}(b)}}{\leq} (0.125 + 0.01)\Phi^{(r-1)}
\end{array}
\end{equation}
Thus, the expected value of the potential function in round $r$ conditioned on its value in the previous round is bounded by
\begin{align*}
\E[\Phi^{(r)} \mid \Phi^{(r-1)}] 
= \E\left[\sum\limits_{i=1}^k w_i^{(r-1)} + \sum\limits_{i\notin G_r} w_i^{(r-1)}\middle| \Phi^{(r-1)}\right] 
\stackrel{(\ref{qbound2})}{\leq} 1.135\Phi^{(r-1)}.
\end{align*}

By the definition of the expected value, this implies that $\E[\Phi^{(r)}]\leq 1.135 \E[\Phi^{(r-1)}]$.
Conditioned on the fact that inequality~(\ref{qbound1}) holds for all rounds, which is true with probability at least $1-\frac{\delta}{4}$, we can conclude that $\E[\Phi^{(t)}]\leq k(1.135)^t$, by induction.
Using Markov's inequality we can state that $\Pr\Big[\Phi^{(t)}\geq \frac{\E[\Phi^{(t)}]}{\delta/2}\Big]\leq \delta/2$.
It follows that with probability at least $1-\frac{\delta}{4}-\frac{\delta}{2}=1-\frac{3\delta}{4}$
\begin{equation}\label{Phibound}
\Phi^{(t)}\leq \frac{2k(1.135)^t}{\delta}.
\end{equation}

We now need a lower bound for $w_i^{(t)}$. Let $m_i^{(r)}$ denote the number of rounds $r'$, up until and including round $r$, for which the procedure \textsc{FastTest} made a mistake and returned $i \in G_{r'}$ although $\text{err}_{D_i}(f^{(r')})> \eps'$.
From Lemma~\ref{lemma:fasttest}($a$), it follows that $\E[m_i^{(r)} - m_i^{(r-1)}] \leq 0.01$ so for $M_i^{(r)}=m_i^{(r)} - 0.01r$ it holds that $\E[M_i^{(r)} \mid M_i^{(r-1)}] \leq M_i^{(r-1)}$.
Therefore, the sequence $\{M_i^{(r)}\}_{r=0}^t$ is a super-martingale. In addition to this, since we can make at most one mistake in each round, it holds that $M_i^{(r)}-M_i^{(r-1)}< 1$. Using the Azuma-Hoeffding inequality with $M_i^{(0)}=m_i^{(0)}-0.01\cdot 0=0$ and the fact that $t\geq150$ we calculate that
\[\Pr\left[m_i^{(t)}\geq 0.18t\right] \leq \exp\Big(-\frac{(0.17t)^2}{2t}\Big)\leq \frac{\delta}{4k}.\]
By union bound, $m_i^{(t)}<0.18t$ holds $\forall i\in [k]$ with probability at least $1-\frac{\delta}{4}$.

The number of times a weight is doubled throughout the algorithm is $\log (w_i^{(t)})$ and it is at least the number of times the error of the classifier was more than $\eps'$ minus the number of times the error was more than $\eps'$ but the \textsc{FastTest} made a mistake, which is exactly $e_i^{(t)}-m_i^{(t)}$.
So $w_i^{(t)} \geq 2^{e_i^{(t)}-m_i^{(t)}}> 2^{e_i^{(t)} -0.18t}$ holds for all $i\in [k]$ with probability at least $1-\frac{\delta}{4}$.
Combining this with the bound from inequality~(\ref{Phibound}) we have that with probability at least $1-\delta$:

$w_i^{(t)}\leq \Phi^{(t)} 
\Rightarrow 2^{e_i^{(t)} -0.18t} < \frac{2k(1.135)^t}{\delta}
\Rightarrow e_i^{(t)} -0.18t < 1+\log \Big(\frac{k}{\delta} \Big) + t\log(1.135)\\
\Rightarrow e_i^{(t)} < 0.18t+\frac{1}{150}t + \frac{1}{150}t + 0.183t < 0.4t$
\end{proof}

It remains to bound the number of samples. \textsc{FastTest} is called $t=150\lceil \log (\frac{k}{\delta}) \rceil$ times, so it requires
$O\Big(\log \Big(\frac{k}{\delta} \Big) \frac{k}{\eps'}\Big) = O\Big(\frac{k}{\eps}\log \Big(\frac{k}{\delta} \Big)\Big)$
samples in total.
The number of samples required to learn each round's classifier is $m_{\eps'/16,\delta'}$, so for all rounds there are required
$O\Big(\log \Big(\frac{k}{\delta} \Big) \frac{1}{\eps'} \Big(d\ln \Big(\frac{1}{\eps'}\Big)+\ln \Big(\frac{1}{\delta'}\Big)\Big)\Big)$ samples. Substituting $\eps'=\eps/6$ and $\delta'=\delta/(4t)=\delta/\left(600\left\lceil\log \left(\frac{k}{\delta}\right)\right\rceil\right)$ we get
$O\Big(\frac{1}{\eps}\log \Big(\frac{k}{\delta} \Big) \Big(d\ln \Big(\frac{1}{\eps}\Big)+ \ln \Big(\frac{\log (k)}{\delta}\Big) \Big)\Big)$
samples in total.
From the addition of the two bounds above, the overall sample complexity bound is:
\[O\Big(\frac{1}{\eps}\ln \Big(\frac{k}{\delta} \Big) \Big(d\ln \Big(\frac{1}{\eps}\Big)+ k+\ln \Big(\frac{1}{\delta}\Big) \Big)\Big)\]
\end{proof}

\section{Sample complexity upper bounds for the non-realizable setting}\label{non_realizable}
We design Algorithms NR1 and NR2 for the non-realizable setting, which generalize the results of Algorithms R1 and R2, respectively. 

\begin{theorem}\label{thNR1}
For any $\eps, \delta \in (0,1)$, $7\eps/6<\alpha<1$, and hypothesis class $\mathcal{F}$ of VC dimension $d$, Algorithm NR1 returns a classifier $f_{\text{NR1}}$ such that $\text{err}_{D_i}(f_{\text{NR1}}) \le (2+\alpha)\OPT+ \eps$ holds for all $i\in[k]$ with probability $1-\delta$ using $m$ samples, where 
\[m = O\Big(\frac{\ln (k)}{\alpha^4 \eps}\Big(d\ln\Big(\frac{1}{\eps}\Big)+k\ln\Big(\frac{k}{\delta}\Big)\Big)\Big).\]
\end{theorem}

\begin{theorem}\label{thNR2}
For any $\eps, \delta \in (0,1)$, $5\eps/4<\alpha<1$, and hypothesis class $\mathcal{F}$ of VC dimension $d$, Algorithm NR2 returns a classifier $f_{\text{NR2}}$ such that $\text{err}_{D_i}(f_{\text{NR2}}) \le (2+\alpha)\OPT+ \eps$ holds for all $i\in[k]$ with probability $1-\delta$ using $m$ samples, where 
\[m = O\Big(\frac{1}{\alpha^4 \eps}\ln\Big(\frac{k}{\delta}\Big)\Big(d\ln\Big(\frac{1}{\eps}\Big)+k\ln\Big(\frac{1}{\alpha}\Big)+\ln\Big(\frac{1}{\delta}\Big)\Big)\Big).\]
\end{theorem}

Their main modification compared to the algorithms in the previous section is that these algorithms use a smoother update rule. Algorithms NR1 and NR2 are the following. 

\begin{algorithm}[H]
\renewcommand{\thealgorithm}{}
\caption{NR1}\label{algNR1}
\begin{algorithmic}[1]
\STATE \textbf{Initialization}: $\forall i\in [k] $ $w_i^{(0)}:=1$; $\alpha' := \alpha/35$; $t:=2\lceil\ln (k) / \alpha'^3 \rceil$; $\eps':=\eps/60$; $\delta':=\delta/(4t)$;
\FOR{$r = 1, \ldots, t$}
\STATE $\tilde{D}^{(r-1)} \gets \frac{1}{\Phi^{(r-1)}}\sum_{i=1}^k \left(w_i^{(r-1)}D_i\right)$, where $\Phi^{(r-1)} := \sum_{i=1}^k w_i^{(r-1)}$;
\STATE Draw a sample set $S^{(r)}$ of size $O\Big(\frac{1}{\alpha'\eps'}\Big(d\ln\Big(\frac{1}{\eps'}\Big)+\ln\Big(\frac{1}{\delta'}\Big)\Big)\Big)$ from $\tilde{D}^{(r-1)}$;
\STATE $f^{(r)} \gets \mathcal{O}_{\mathcal{F}}(S^{(r)})$;
\FOR{$i=1,\ldots, k$}
\STATE Draw a sample set $T_i$ of size $O\Big(\frac{1}{\alpha'\eps'} \ln \Big( \frac{k}{\delta'}\Big)\Big)$ from $D_i$;
\STATE $s_i^{(r)} \gets \min\left(\frac{\text{err}_{T_i}(f^{(r)})\alpha'^2}{(1+3\alpha')\text{err}_{S^{(r)}}(f^{(r)}) + 3\eps'}, \alpha'\right)$
\STATE \textbf{Update}: $w_i^{(r)} \gets w_i^{(r-1)} (1 + s_i^{(r)})$
\ENDFOR
\ENDFOR
\STATE \RETURN $f_{\text{NR1}} = \maj(\{f^{(r)}\}_{r=1}^t)$;
\end{algorithmic}
\end{algorithm}

\begin{algorithm}[H]
\renewcommand{\thealgorithm}{}
\caption{NR2}\label{algNR2}
\begin{algorithmic}[1]
\STATE \textbf{Initialization}: $\forall i\in [k] $ $w_i^{(0)}:=1$; $\alpha' := \alpha/40$; $t:=2\lceil\ln ({4k}/\delta) / \alpha'^3 \rceil$; $\eps':=\eps/64$; $\delta':=\delta/(4t)$;
\FOR{$r = 1, \ldots, t$}
\STATE $\tilde{D}^{(r-1)} \gets \frac{1}{\Phi^{(r-1)}}\sum_{i=1}^k \left(w_i^{(r-1)}D_i\right)$, where $\Phi^{(r-1)} := \sum_{i=1}^k w_i^{(r-1)}$;
\STATE Draw a sample set $S^{(r)}$ of size $O\Big(\frac{1}{\alpha'\eps'}\Big(d\ln\Big(\frac{1}{\eps'}\Big)+\ln\Big(\frac{1}{\delta'}\Big)\Big)\Big)$ from $\tilde{D}^{(r-1)}$;
\STATE $f^{(r)} \gets \mathcal{O}_{\mathcal{F}}(S^{(r)})$;
\FOR{$i=1,\ldots, k$}
\STATE Draw a sample set $T_i$ of size $O\Big(\frac{1}{\alpha'\eps'} \ln \Big( \frac{1}{\alpha'}\Big)\Big)$ from $D_i$;
\STATE $s_i^{(r)} \gets \min\left(\frac{\text{err}_{T_i}(f^{(r)})\alpha'^2}{(1+3\alpha')\text{err}_{S^{(r)}}(f^{(r)}) + 3\eps'}, \alpha'\right)$
\STATE \textbf{Update}: $w_i^{(r)} \gets w_i^{(r-1)} (1 + s_i^{(r)})$
\ENDFOR
\ENDFOR
\STATE \RETURN $f_{\text{NR2}} = \maj(\{f^{(r)}\}_{r=1}^t)$;
\end{algorithmic}
\end{algorithm}

The algorithms of this section share many useful properties and the proofs of their corresponding
theorems follow similar steps. We will first prove some of these shared properties.

\begin{corollary*}[of Lemma~\ref{lemma:chernoff}]
If $X$ is the average of $n$ independent random variables taking values in $\{0, 1\}$, then:
\begin{equation}\label{chernoff4} 
\Pr[X \leq (1-\alpha)\E[X]-\eps] \leq \exp(-\alpha\eps n) \text{ } \forall \alpha,\eps\in(0,1)
\end{equation}
\begin{equation}\label{chernoff5} 
\Pr[X \ge (1+\alpha)\E[X]+\eps] \leq \exp\Big(-\frac{\alpha\eps n}{3}\Big) \text{ } \forall \alpha,\eps\in(0,1)
\end{equation}
\end{corollary*}
\begin{proof}
We first prove inequality (\ref{chernoff4}). Note that if $\E[X] \le \eps$ then the inequality is trivially true so we only need to consider $\E[X]>\eps$. Let $s = \alpha + \frac{\eps}{\E[X]}$. Notice that $s^2 \ge \frac{2\alpha\eps}{\E[X]}$. Thus, by inequality (\ref{chernoff1}),
\[\Pr[X \leq (1-\alpha)\E[X]-\eps] \leq \exp(-s^2 \E[X] n/2) \le \exp(-\alpha\eps n).\]
Next we prove inequality (\ref{chernoff5}). Again let $s = \alpha + \frac{\eps}{\E[X]}$. If $s < 1$ then by inequality (\ref{chernoff2},
\[\Pr[X \ge (1+\alpha)\E[X]+\eps] \leq \exp(-s^2 \E[X] n/3) \le \exp(-2\alpha\eps n/3).\]
If $s \ge 1$ then by inequality (\ref{chernoff3}),
\[\Pr[X \ge (1+\alpha)\E[X]+\eps] \leq \exp(-s \E[X] n/3) \le \exp(-\eps n/3) \le \exp(-\alpha\eps n/3).\]
\end{proof}

Lemma \ref{lemma:VC} proves that the error of the classifier $f^{(r)}$ of each round on the weighted mixture of distributions is low. It holds due to a known extension of the VC Theorem and Chernoff bounds, but we prove it here for our parameters for completeness.\\

\begin{lemma}\label{lemma:VC}
With probability at least $1-\delta/2$, for all rounds $r\in[t]$:
\begin{enumerate}[label=(\alph*)]
\item $(1+3\alpha')\text{err}_{S^{(r)}}(f^{(r)}) + 3\eps' \leq (1+7\alpha')\OPT + 19\eps'$.
\item $\text{err}_{\tilde{D}^{(r-1)}}(f^{(r)}) \leq (1+\alpha')\text{err}_{S^{(r)}}(f^{(r)}) +\eps'$.
\end{enumerate}
\end{lemma}
\begin{proof}
Let $S^{(r)}$ be a set of samples of size $C \cdot \frac{1}{\alpha'\eps'}\Big(d\ln\Big(\frac{1}{\eps'}\Big) + \ln\Big(\frac{1}{\delta'}\Big)\Big)$ drawn from $\tilde{D}^{(r-1)}$, where $C$ is a constant. We will prove that for large enough constant $C$ the two statements hold simultaneously for all rounds, each with probability at least $1-t\delta'$. It suffices to prove that each statement in each round holds with probability at least $1-\delta'$. For a given round $r$:

\begin{enumerate}[label=(\alph*)]
\item
By $f^*$'s definition it holds that $\text{err}_{D_i}(f^*) \leq \OPT + \eps'$ $\forall i\in [k]$, so it must also hold that $\text{err}_{\tilde{D}^{(r-1)}}(f^*) \leq \OPT + \eps'$, since $\tilde{D}^{(r-1)}$ is a weighted average of the distributions. From the Corollary it holds that
$\Pr[\text{err}_{S^{(r)}}(f^*) \ge (1+\alpha')\text{err}_{\tilde{D}^{(r-1)}}(f^*) +\eps'] \le \exp(-\alpha'\eps' |S^{(r)}| / 3) \le \delta'$
and since $\alpha'\le 1$, it is easy to see that with probability at least $1-\delta'$,
\begin{equation}\label{f_star_error}
\text{err}_{S^{(r)}}(f^*) \leq (1+\alpha')\OPT + 3\eps'
\end{equation}

Since $f^{(r)}$ is the error minimizing classifier for the sample $S^{(r)}$, it holds that $\text{err}_{S^{(r)}}(f^{(r)}) \leq \text{err}_{S^{(r)}}(f^*) + \eps'$. Therefore,
\[(1+3\alpha')\text{err}_{S^{(r)}}(f^{(r)}) + 3\eps' \leq (1+3\alpha')\text{err}_{S^{(r)}}(f^*) + 7\eps' \stackrel{(\ref{f_star_error})}{\leq} (1+7\alpha')\OPT + 19\eps'.\]

\item We prove the second statement for all $f\in\mathcal{F}$, using Theorem 5.7 from~\cite{AB09}. The theorem states that for every $h\in\mathcal{H}$, it holds that $\text{err}_{D}(h) \leq (1+\gamma)\text{err}_{S}(h) + \beta$ with probability at least $1-4\Pi_{\mathcal{H}}(2m) \exp\Big(\frac{-\gamma\beta m}{4(\gamma+1)}\Big)$, where $S$ is a sample of size $m$ drawn from a distribution $D$ on $\mathcal{X}\times\{0,1\}$, $\gamma > 2\beta$, and $\Pi_{\mathcal{H}}(n) = \max \{ \vert \mathcal{H}_{\vert S}\vert : S \subseteq \mathcal X \text{ and } \vert S \vert=n\}$ is the growth function of $\mathcal{H}$. 

We apply Theorem 5.7 for $\gamma=\alpha'$, $\beta = \eps'$, $D = \tilde{D}^{(r-1)}$, $S = S^{(r)}$, $\mathcal{H} = \mathcal{F}$. Since the VC-dimension of $\mathcal{F}$ is $d$, from [\cite{AB09}, Theorem 3.7] it holds that $\Pi_{\mathcal{F}}(2m) \leq \Big(\frac{2em}{d}\Big)^d$. In our setting, the theorem states that, given round $r$, for every $f\in\mathcal{F}$, it holds that $\text{err}_{\tilde{D}^{(r-1)}}(f) \leq (1+\alpha')\text{err}_{S^{(r)}}(f) + \eps'$ with probability at least $1-4\Big(\frac{2em}{d}\Big)^d\exp\Big(\frac{-\alpha'\eps' m}{4(\alpha'+1)}\Big)$.

It remains to prove that, for large enough $C$, $m = C \cdot \frac{1}{\alpha'\eps'}\Big(d\ln\Big(\frac{1}{\eps'}\Big) + \ln\Big(\frac{1}{\delta'}\Big)\Big)$ samples suffice to guarantee that $4\Big(\frac{2em}{d}\Big)^d\exp\Big(\frac{-\alpha'\eps' m}{4(\alpha'+1)}\Big) \leq \delta'$ so that the statement holds with probability at least $1-\delta'$.
It suffices to prove that for the given $m$:

$\ln(4)+d\ln(2e)+d\ln\Big(\frac{m}{d}\Big)-\frac{\alpha'}{8}\eps'm \leq -\ln\Big(\frac{1}{\delta'}\Big) \\
\Leftrightarrow 
\ln(4) +d\ln(2e) + d\ln\Big(\frac{m}{d}\Big) + \ln\Big(\frac{1}{\delta'}\Big) \leq \frac{C}{8}d\ln\Big(\frac{1}{\eps'}\Big) + \frac{C}{8}\ln\Big(\frac{1}{\delta'}\Big)$.

We consider two cases:
\begin{enumerate}[label=\roman*.]
\item If $d\ln\Big(\frac{1}{\eps'}\Big) \geq \ln\Big(\frac{1}{\delta'}\Big)$, then $\frac{m}{d} \leq \frac{2C}{\alpha'\eps'}\ln\Big(\frac{1}{\eps'}\Big) <  \frac{C}{\eps'^2}\ln\Big(\frac{1}{\eps'}\Big)$. So to prove the statement, it suffices to prove that
\[\ln(4) +d\ln(2e) +d\Big(\ln(C) + 2\ln\Big(\frac{1}{\eps'}\Big)+\ln\ln\Big(\frac{1}{\eps'}\Big)\Big) + \ln\Big(\frac{1}{\delta'}\Big)
 \leq \frac{C}{8}d\ln\Big(\frac{1}{\eps'}\Big) + \frac{C}{8}\ln\Big(\frac{1}{\delta'}\Big).\]
The latter inequality holds for large enough $C$.

\item If $d\ln\Big(\frac{1}{\eps'}\Big) \leq \ln\Big(\frac{1}{\delta'}\Big)$, then $\frac{m}{d} \leq \frac{2C}{\alpha'\eps'}\frac{\ln(1/\delta')}{d} < \frac{C}{\eps'^2}\frac{\ln(1/\delta')}{d}$. So to prove the statement, it suffices to prove that
\[\ln(4) +d\ln(2e) +d\Big(\ln(C) + 2\ln\Big(\frac{1}{\eps'}\Big) +\ln\Big(\frac{\ln(1/\delta')}{d}\Big)\Big) + \ln\Big(\frac{1}{\delta'}\Big)
\leq \frac{C}{8}d\ln\Big(\frac{1}{\eps'}\Big) + \frac{C}{8}\ln\Big(\frac{1}{\delta'}\Big).\]
If we prove that $d\ln\Big(\frac{\ln(1/\delta')}{d}\Big) \leq \ln(1/\delta')$, then the inequality holds for large enough $C$. Indeed, it holds that ${\ln\Big(\frac{\ln(1/\delta')}{d}\Big)}/{\frac{\ln(1/\delta')}{d}} \leq \frac{1}{e}$, since $\max_{x\in\mathbb{R}}\{\ln(x)/x\} = \frac{1}{e}$.
\end{enumerate}
Thus the second statement holds too with probability at least $1-\delta'$.
\end{enumerate}
\end{proof}
Lemmas \ref{empirical_sum} and \ref{s_sum} give us two inequalities that are useful for all the proofs of Section 4.

\begin{lemma}\label{empirical_sum}
Let $L_r = \{i\in[k]  \mid |\text{err}_{T_i}(f^{(r)}) -\text{err}_{D_i}(f^{(r)})| \leq \alpha' \cdot \text{err}_{D_i}(f^{(r)})+\eps'\}$. With probability $1-\delta/2$, it holds that
\[ \sum\limits_{i\in L_r} \left(w_i^{(r-1)} \text{err}_{T_i} (f^{(r)})\right) \leq [(1+3\alpha')\text{err}_{S^{(r)}}(f^{(r)}) + 3\eps']\Phi^{(r-1)} \leq [(1+7\alpha')\OPT + 19\eps']\Phi^{(r-1)}.\]
\end{lemma}
\begin{proof}
By linearity of expectation,
\begin{align*}
\text{err}_{\tilde{D}^{(r-1)}}(f^{(r)}) 
&=\frac{1}{\Phi^{(r-1)}}\sum\limits_{i=1}^k \left(w_i^{(r-1)}\text{err}_{D_i}(f^{(r)})\right)\\
&\geq \frac{1}{\Phi^{(r-1)}}\sum\limits_{i\in L_r} \left(w_i^{(r-1)}\text{err}_{D_i}(f^{(r)})\right)\\
&\geq \frac{1}{(1+\alpha')\Phi^{(r-1)}} \sum\limits_{i\in L_r} \left(w_i^{(r-1)}\text{err}_{T_i}(f^{(r)})\right) - \frac{\eps'}{1+\alpha'}.
\end{align*}
Therefore, $\sum\limits_{i\in L_r} \left(w_i^{(r-1)}\text{err}_{T_i}(f^{(r)})\right) \leq [(1+\alpha') \text{err}_{\tilde{D}^{(r-1)}}(f^{(r)}) + \eps']\Phi^{(r-1)}$. By Lemma~\ref{lemma:VC}(b), it follows that with probability $1-\delta/2$,
 \begin{equation*}
 \begin{array}{ll@{}ll}
 \sum\limits_{i\in L_r} \left(w_i^{(r-1)}\text{err}_{T_i}(f^{(r)})\right) & \leq [(1+\alpha')(1+\alpha')\text{err}_{S^{(r)}}(f^{(r)}) + (1+\alpha')\eps' + \eps']\Phi^{(r-1)} \\
 & \leq [(1+3\alpha')\text{err}_{S^{(r)}}(f^{(r)}) + 3\eps']\Phi^{(r-1)}\\
 & \stackrel{\text{Lemma~\ref{lemma:VC}(a)}}{\leq}[(1+7\alpha')\OPT + 19\eps']\Phi^{(r-1)}.
 \end{array}
\end{equation*}
\end{proof}
\begin{lemma}\label{s_sum}
For all $i\in [k]$ it holds that
\[ \sum\limits_{r=1}^t s_i^{(r)} \leq \frac{\ln(\Phi^{(t)})}{1-\alpha'/2}.\]
\end{lemma}
\begin{proof}
In every round $r$, $w_i^{(r)} = w_i^{(r-1)}(1+s_i^{(r)})$. Therefore for any $i\in [k]$,
\begin{align*}
w_i^{(t)} &= \prod_{r=1}^t (1+s_i^{(r)})\\
&\ge \prod_{r=1}^t \exp(s_i^{(r)} - (s_i^{(r)})^2/2)\\
&\stackrel{s_i^{(r)}\le a'}{\ge} \exp\left((1-\alpha'/2)\sum_{r=1}^t s_i^{(r)}\right),
\end{align*}
where the second to last inequality holds since $(1+x)\ge \exp(x-x^2/2)$ for $x\in\mathbb{R_+}$.
The inequality follows since $w_i^{(t)}\le \Phi^{(t)}$ for all $i\in [k]$.
\end{proof}

We will now give the proof of Theorem~\ref{thNR1}.
\begin{proof}[Proof of Theorem~\ref{thNR1}]
By the Corollary, for a given round $r$ and player $i$,
\[\Pr[|\text{err}_{T_i}(f^{(r)}) -\text{err}_{D_i}(f^{(r)})| \ge \alpha' \cdot \text{err}_{D_i}(f^{(r)})+\eps'] \le 2\exp(-\alpha'\eps' |T_i| / 3).\]
If $\vert T_i \vert =\frac{3}{\eps'\alpha'}\ln\Big(\frac{k}{\delta'}\Big) = O\Big(\frac{1}{\eps'\alpha'}\ln\Big(\frac{k}{\delta'}\Big)\Big)$, the inequality 
\begin{equation}\label{empirical_error}
|\text{err}_{T_i}(f^{(r)}) -\text{err}_{D_i}(f^{(r)})| \le \alpha' \cdot \text{err}_{D_i}(f^{(r)})+\eps'
\end{equation}
holds with probability at least $1-2\delta'/k$. By union bound, it follows that (\ref{empirical_error}) holds for every $i$ and every $r$ with probability at least $1-2\delta't=1-\delta/2$.

With probability at least $1-\delta$ inequality (\ref{empirical_error}) and the inequality of Lemma~\ref{empirical_sum} hold for  all rounds and players. We restrict the rest of the proof to this event.
It holds that,
\begin{align*}
\Phi^{(r)} &= \Phi^{(r-1)} + \sum_{i=1}^k \left(w_i^{(r-1)} \cdot s_i^{(r)}\right) \\
&\le \Phi^{(r-1)} + \frac{\alpha'^2}{(1+3\alpha')\text{err}_{S^{(r)}}(f^{(r)}) + 3\eps'}\sum_{i=1}^k \left(w_i^{(r-1)} \text{err}_{T_i}(f^{(r)})\right) \\
&\stackrel{L_r=[k]}{\leq} \Phi^{(r-1)}\left(1+\frac{\alpha'^2}{(1+3\alpha')\text{err}_{S^{(r)}}(f^{(r)}) + 3\eps'}[(1+3\alpha')\text{err}_{S^{(r)}}(f^{(r)}) + 3\eps']\right)\\
&= \Phi^{(r-1)}(1+\alpha'^2)
\end{align*}

By induction, $\Phi^{(t)}\le \Phi^{(0)}(1+\alpha'^2)^t  = k(1+\alpha'^2)^t \le k\exp(t\alpha'^2)$.
From Lemma~\ref{s_sum} and $t=2\lceil \ln(k)/\alpha'^3\rceil$, it follows that 
\begin{equation}\label{s_sum3}
\sum_{r=1}^t s_i^{(r)} \le \frac{\ln (k) + t\alpha'^2}{1-\alpha'/2} \le \frac{1+\alpha'}{1-\alpha'/2}t\alpha'^2.
\end{equation}

Let $G_i$ be the set of rounds $r$ such that $s_i^{(r)} < \alpha'$. We consider these to be the ``good'' classifiers. Because of (\ref{s_sum3}), we have $|[t] \setminus G_i| \le \frac{1}{\alpha'}\sum_{r\in[t] \setminus G_i} \alpha' \leq \frac{1}{\alpha'}\sum_{r=1}^t s_i^{(r)} \le \frac{1+\alpha'}{1-\alpha'/2}\alpha' t$. For the classifiers of the rounds $r\in G_i$, it holds that
\[\sum_{r\in G_i} \frac{\text{err}_{T_i}(f^{(r)}) \alpha'^2}{(1+3\alpha')\text{err}_{S^{(r)}}(f^{(r)}) + 3\eps'} = \sum_{r\in G_i} s_i^{(r)} \leq \sum_{r=1}^t s_i^{(r)} \stackrel{(\ref{s_sum3})}{\leq} \frac{1+\alpha'}{1-\alpha'/2}\alpha'^2 t.\]

Thus, $\sum_{r\in G_i} \text{err}_{T_i}(f^{(r)}) \stackrel{\ref{lemma:VC}(a)}{\leq} t\frac{1+\alpha'}{1-\alpha'/2}[(1+7\alpha')\OPT + 19\eps']$. From inequality~(\ref{empirical_error}), it follows that:

\begin{align*}
&(1-\alpha')\sum_{r\in G_i} \text{err}_{D_i}(f^{(r)}) - |G_i | \eps' \leq  t\frac{1+\alpha'}{1-\alpha'/2}[(1+7\alpha')\OPT + 19\eps'] \\
&\Rightarrow \sum_{r\in G_i} \text{err}_{D_i}(f^{(r)}) \leq t\frac{1+\alpha'}{(1-\alpha'/2)(1-\alpha')}[(1+7\alpha')\OPT + 19\eps'] + \frac{t\eps'}{1-\alpha'}\\
&\Rightarrow \sum_{r\in G_i} \text{err}_{D_i}(f^{(r)}) \leq [(1+12\alpha')\OPT+25\eps']t,
\end{align*}
which holds for $\alpha'<1/12$.

For each example $e$ that is a mistake for $f_{\text{NR1}}$, it must be a mistake for at least $t/2 - |[t] \setminus G_i|$ members of $G_i$. Thus the fraction of error of $f_{\text{NR1}}$ is at most 
\[\frac{\sum_{r\in G_i} \text{err}_{D_i}(f^{(r)})}{t/2 - |[t] \setminus G_i|} \le \frac{(1+12\alpha')\OPT+25\eps'}{1/2 - (1+\alpha')\alpha' /(1-\alpha'/2)} \le (2+35\alpha')\OPT + 60\eps'.\]
Having set $\alpha'=\alpha/35$ and $\eps'=\eps/60$ we get that $\text{err}_{D_i}(f_{\text{NR1}}) \leq (2+\alpha)\OPT+\eps$.

As for the total number of samples, it is the sum of $O(\frac{k}{\alpha'\eps'}\ln(k/\delta'))$ and $O\Big(\frac{1}{\alpha'\eps'} \Big(d\ln \Big(\frac{1}{\eps'}\Big)+\ln \Big(\frac{1}{\delta'}\Big)\Big)\Big)$ samples for each round.
Because there are $O(\ln (k) / \alpha'^3)$ rounds, the total number of samples is
\[O\Big(\frac{\ln (k)}{\alpha'^4\eps'} \Big(k\ln\left(\frac{k}{\delta'}\right)+d\ln \Big(\frac{1}{\eps'}\Big)\Big)\Big)= O\Big(\frac{\ln (k)}{\alpha^4\eps} \Big(k\ln\left(\frac{k}{\delta}\right)+d\ln \Big(\frac{1}{\eps}\Big)\Big)\Big).\]
\end{proof}

Algorithm NR2 faces a similar challenge as Algorithm R2. Given a player $i$, since the number of samples $T_i$ used to estimate $\text{err}_{D_i} (f^{(r)})$ in each round is low, the estimation is not very accurate. Ideally, we would want the inequality \[|\text{err}_{T_i} (f^{(r)})-\text{err}_{D_i} (f^{(r)})| \leq \alpha' \cdot \text{err}_{D_i} (f^{(r)})+\eps'\] to hold for all players and all rounds with high probability. The ``good'' classifiers are now defined as the ones corresponding to rounds for which the inequality holds and $\text{err}_{T_i} (f^{(r)})$ is not very high (an indication of which is that $s_i^{(r)} < \alpha'$). The expected number of rounds that either one of these properties does not hold is a constant fraction of the rounds ($\approx t\alpha'$) and due to the high number of rounds it is concentrated around that value, as in Algorithm R2. The proof of Theorem~\ref{thNR2} is the following.

\begin{proof}[Proof of Theorem~\ref{thNR2}]
By the Corollary, for a given round $r$ and player $i$,
\[\Pr[|\text{err}_{T_i}(f^{(r)}) -\text{err}_{D_i}(f^{(r)})| \ge \alpha' \cdot \text{err}_{D_i}(f^{(r)})+\eps'] \le 2\exp(-\alpha'\eps' |T_i| / 3).\]
If $\vert T_i \vert=\frac{6}{\eps'\alpha'}\ln\Big(\frac{\sqrt{2}}{\alpha'}\Big) = O\Big(\frac{1}{\eps'\alpha'}\ln\Big(\frac{1}{\alpha'}\Big)\Big)$, then 
\begin{equation}\label{empirical4}
\Pr[|\text{err}_{T_i}(f^{(r)}) -\text{err}_{D_i}(f^{(r)})| \ge \alpha' \cdot \text{err}_{D_i}(f^{(r)})+\eps'] \le \alpha'^2.
\end{equation}

Assuming that the inequality of Lemma~\ref{empirical_sum} holds, which is true with probability $1-\delta/2$, it follows that
\begin{align*}
&\E[\Phi^{(r)} \mid \Phi^{(r-1)}]\\
&\le \E\left[\Phi^{(r-1)} + \frac{\alpha'^2}{(1+3\alpha')\text{err}_{S^{(r)}}(f^{(r)}) + 3\eps'}\sum_{i\in L_r} \left(w_i^{(r-1)} \text{err}_{T_i}(f^{(r)})\right) + \sum_{i\notin L_r} \left(w_i^{(r-1)}s_i^{(r-1)}\right) \middle| \Phi^{(r-1)}\right]\\
& \le \E\left[\Phi^{(r-1)} + \frac{\alpha'^2}{(1+3\alpha')\text{err}_{S^{(r)}}(f^{(r)}) + 3\eps'}[(1+3\alpha')\text{err}_{S^{(r)}}(f^{(r)}) + 3\eps']\Phi^{(r-1)} + \alpha'\sum_{i\notin L_r} w_i^{(r-1)} \middle| \Phi^{(r-1)}\right]\\
& \stackrel{(\ref{empirical4})}{\le} \Phi^{(r-1)}(1+\alpha'^2+\alpha'^3)
\end{align*}

By the definition of expectation, $\E[\Phi^{(r)}] \leq  \E[\Phi^{(r-1)}](1+\alpha'^2+\alpha'^3)$. So by induction and the fact that $\Phi^{(0)}=k$, $\E[\Phi^{(t)}]\le k\exp(t\alpha'^2(1+\alpha'))$. 
Markov's inequality states that $\Pr[\Phi^{(t)} \ge \frac{\E[\Phi^{(t)}]}{\delta/4}] \leq \delta/4$. So with overall probability $1-\delta/4-\delta/2 = 1-3\delta/4$ it holds that $\Phi^{(t)} \le \frac{4k}{\delta}\exp(t\alpha'^2(1+\alpha'))$.

From Lemma~\ref{s_sum} and $t=2\lceil \ln (4k/\delta)/\alpha'^3 \rceil$, it follows that 
\begin{equation}\label{s_sum4}
\sum_{r=1}^t s_i^{(r)} \le \frac{\ln (4k/\delta) + t\alpha'^2(1+\alpha')}{1-\alpha'/2} \le \frac{(1+2\alpha')}{1-\alpha'/2}t\alpha'^2.
\end{equation}

For $G_i=\{r\in[t] \mid s_i^{(r)} < \alpha' \}$, we have $|[t] \setminus G_i| \le \frac{1+2\alpha'}{1-\alpha'/2}\alpha' t$ because of (\ref{s_sum4}).

Let $R_i=\{r\in[t] \mid |\text{err}_{T_i}(f^{(r)}) -\text{err}_{D_i}(f^{(r)})| \leq \alpha' \cdot \text{err}_{D_i}(f^{(r)})+\eps'\}$. For the classifiers of the rounds $r\in G_i\cap R_i$:

\begin{align*}
\sum_{r\in G_i\cap R_i} \text{err}_{D_i}(f^{(r)}) 
&\leq \sum_{r\in G_i\cap R_i}\frac{\text{err}_{T_i}(f^{(r)})}{1-\alpha'} + \frac{|G_i\cap R_i|\eps'}{1-\alpha'}\\
&\leq \sum_{r\in G_i\cap R_i} \frac{(1+3\alpha')\text{err}_{S^{(r)}}(f^{(r)}) + 3\eps'}{\alpha'^2} \frac{\text{err}_{T_i}(f^{(r)}) \alpha'^2}{(1-\alpha')[(1+3\alpha')\text{err}_{S^{(r)}}(f^{(r)}) + 3\eps']} + \frac{t\eps'}{1-\alpha'}\\
&= \sum_{r\in G_i\cap R_i}\frac{(1+3\alpha')\text{err}_{S^{(r)}}(f^{(r)}) + 3\eps'}{(1-\alpha')\alpha'^2} s_i^{(r)} + \frac{t\eps'}{1-\alpha'}\\
&\stackrel{(\ref{s_sum4})}{\leq} \frac{(1+7\alpha')\OPT+19\eps'}{(1-\alpha')\alpha'^2}\frac{(1+2\alpha')}{1-\alpha'/2}t\alpha'^2 + \frac{t\eps'}{1-\alpha'}\\
& \leq [(1+15\alpha')\OPT+25\eps']t
\end{align*}
which holds for $\alpha'<1/15$.

We will now bound $|[t]\setminus R_i|$. For every round $r$, let $m^{(r)}$ be the indicator random variable of the set $[t]\setminus R_i$ and let $y^{(r)}=\alpha'^2$. It holds that for all rounds $r$, $|m^{(r)}-y^{(r)}| \le 1$ and $m^{(r)}, y^{(r)} \ge 0$. In addition, from inequality~(\ref{empirical4}) it follows that $\E[m^{(r)}-y^{(r)} \mid \sum_{r'<r}m^{(r')}, \sum_{r'<r}y^{(r')}] = \alpha'^2-\alpha'^2 \le 0$.

Using [\cite{KY14}, Lemma 10], with $\varepsilon=1/2$ and $A=\alpha'^2$, we get that \[\Pr\left[\sum_{r=1}^t m^{(r)} \geq 2\alpha'^2t + 2\alpha'^2t\right] \leq \exp(-\alpha'^2t/2) \leq \delta/4k.\]

So $|[t]\setminus R_i| = \sum_{r=1}^t m^{(r)} \leq 4\alpha'^2t$ for all $i$ with probability at least $1-\delta/4$, by union bound.

For each example $e$ that is a mistake for $f_{\text{NR2}}$, it must be a mistake for at least $t/2 - |[t] \setminus (G_i\cap R_i)|$ members of $G_i\cap R_i$. Thus, with probability at least $1-\delta$, the fraction of error of $f_{\text{NR2}}$ is at most
\[\frac{\sum_{r\in G_i\cap R_i} \text{err}_{D_i}(f^{(r)})}{t/2 - |[t]\setminus (G_i\cap R_i)| }\leq \frac{(1+15\alpha')\OPT+25\eps'}{t/2 - 4\alpha'^2t - (1+\alpha')\alpha' t/(1-\alpha'/2)} \le (2+40\alpha')\OPT + 64\eps'.\]
Having set $\alpha'=\alpha/40$ and $\eps'=\eps/64$ we get that $\text{err}_{D_i}(f_{\text{NR2}}) \leq (2+\alpha)\OPT+\eps$.

As for the total number of samples, it is the sum of $O(\frac{k}{\alpha'\eps'}\ln(1/\alpha'))$ samples and $O\Big(\frac{1}{\alpha'\eps'} \Big(d\ln \Big(\frac{1}{\eps'}\Big)+\ln \Big(\frac{1}{\delta'}\Big)\Big)\Big)$ samples for each round.
Because there are $O(\ln(k/\delta) / \alpha'^3)$ rounds, the total number of samples is
\[O\Big(\frac{1}{\alpha^4\eps} \ln \Big(\frac{k}{\delta}\Big) \Big(k\ln\left(\frac{1}{\alpha}\right)+d\ln \Big(\frac{1}{\eps}\Big) + \ln\Big(\frac{1}{\delta}\Big)\Big)\Big).\]
\end{proof}

We note that the classifiers returned by these algorithms have a multiplicative approximation factor of almost $2$ on the error. A different approach would be to allow for randomized classifiers with low error probability over both the randomness of the example and the classifier. We design two algorithms, NR1-AVG and NR2-AVG that return a classifier which satisfies this form of guarantee on the error without the $2$-approximation factor but use roughly $\frac{\alpha}{\eps}$ times more samples. The returned classifier is a randomized algorithm that, given an element $x$, chooses one of the classifiers of all rounds uniformly at random and returns the label that this classifier gives to $x$. For any distribution over examples, the error probability of this randomized classifier is exactly the average of the error probability of classifiers $f^{(1)}, f^{(2)}, \ldots, f^{(t)}$, hence the AVG in the names. The guarantees of the algorithms are stated in the next two theorems.

\begin{theorem}\label{thNR1-AVG}
For any $\eps, \delta \in(0,1)$, $24\eps/25<\alpha<1$, and hypothesis class $\mathcal{F}$ of VC dimension $d$, Algorithm NR1-AVG returns a classifier $f_{\text{NR1-AVG}}$ such that for the expected error $\overline{\text{err}}_{D_i}(f_{\text{NR1-AVG}}) \le (1+\alpha)\OPT+ \eps$ holds for all $i\in[k]$ with probability $1-\delta$ using $m$ samples, where 
\[m = O\Big(\frac{\ln (k)}{\alpha^3 \eps^2}\Big(d\ln\Big(\frac{1}{\eps}\Big)+k\ln\Big(\frac{k}{\delta}\Big)\Big)\Big).\]
\end{theorem}

\begin{theorem}\label{thNR2-AVG}
For any $\eps, \delta  \in(0,1)$, $30\eps/29<\alpha<1$, and hypothesis class $\mathcal{F}$ of VC dimension $d$, Algorithm NR2-AVG returns a classifier $f_{\text{NR2-AVG}}$ such that for the expected error $\overline{\text{err}}_{D_i}(f_{\text{NR2-AVG}}) \le (1+\alpha)\OPT+ \eps$ holds for all $i\in[k]$ with probability $1-\delta$ using $m$ samples, where 
\[m = O\Big(\frac{1}{\alpha^3 \eps^2}\ln\Big(\frac{k}{\delta}\Big)\Big((d+k)\ln\Big(\frac{1}{\eps}\Big)+\ln\Big(\frac{1}{\delta}\Big)\Big)\Big).\]
\end{theorem}

Algorithms NR1-AVG and NR2-AVG are the following.

\begin{algorithm}[H]
\renewcommand{\thealgorithm}{}
\caption{NR1-AVG}\label{algNR1-AVG}
\begin{algorithmic}[1]
\STATE \textbf{Initialization}: $\forall i\in [k] $ $w_i^{(0)}:=1$; $\alpha' := \alpha/12$; $t:=2\lceil\ln (k) / (\eps'\alpha'^2) \rceil$; $\eps':=\eps/25$; $\delta':=\delta/(4t)$;
\FOR{$r = 1, \ldots, t$}
\STATE $\tilde{D}^{(r-1)} \gets \frac{1}{\Phi^{(r-1)}}\sum_{i=1}^k \left(w_i^{(r-1)}D_i\right)$, where $\Phi^{(r-1)} := \sum_{i=1}^k w_i^{(r-1)}$;
\STATE Draw a sample set $S^{(r)}$ of size $O\Big(\frac{1}{\alpha'\eps'}\Big(d\ln\Big(\frac{1}{\eps'}\Big)+\ln\Big(\frac{1}{\delta'}\Big)\Big)\Big)$ from $\tilde{D}^{(r-1)}$;
\STATE $f^{(r)} \gets \mathcal{O}_{\mathcal{F}}(S^{(r)})$;
\FOR{$i=1,\ldots, k$}
\STATE Draw a sample set $T_i$ of size $O\Big(\frac{1}{\alpha'\eps'} \ln \Big( \frac{k}{\delta'}\Big)\Big)$ from $D_i$;
\STATE $s_i^{(r)} \gets \frac{\text{err}_{T_i}(f^{(r)})\eps'\alpha'}{(1+3\alpha')\text{err}_{S^{(r)}}(f^{(r)}) + 3\eps'}$
\STATE \textbf{Update}: $w_i^{(r)} \gets w_i^{(r-1)} (1 + s_i^{(r)})$
\ENDFOR
\ENDFOR
\STATE \RETURN $f_{\text{NR1-AVG}}$, where $f_{\text{NR1-AVG}}(x) \stackrel{R}{\gets} \{f^{(r)}(x)\}_{r=1}^t$;
\end{algorithmic}
\end{algorithm}

\begin{algorithm}[H]
\renewcommand{\thealgorithm}{}
\caption{NR2-AVG}\label{algNR2-AVG}
\begin{algorithmic}[1]
\STATE \textbf{Initialization}: $\forall i\in [k] $ $w_i^{(0)}:=1$; $\alpha' := \alpha/15$; $t:=2\lceil\ln (4k/\delta) / (\eps'\alpha'^2) \rceil$; $\eps':=\eps/29$; $\delta':=\delta/(4t)$;
\FOR{$r = 1, \ldots, t$}
\STATE $\tilde{D}^{(r-1)} \gets \frac{1}{\Phi^{(r-1)}}\sum_{i=1}^k \left(w_i^{(r-1)}D_i\right)$, where $\Phi^{(r-1)} := \sum_{i=1}^k w_i^{(r-1)}$;
\STATE Draw a sample set $S^{(r)}$ of size $O\Big(\frac{1}{\alpha'\eps'}\Big(d\ln\Big(\frac{1}{\eps'}\Big)+\ln\Big(\frac{1}{\delta'}\Big)\Big)\Big)$ from $\tilde{D}^{(r-1)}$;
\STATE $f^{(r)} \gets \mathcal{O}_{\mathcal{F}}(S^{(r)})$;
\FOR{$i=1,\ldots, k$}
\STATE Draw a sample set $T_i$ of size $O\Big(\frac{1}{\alpha'\eps'} \ln \Big( \frac{1}{\eps'}\Big)\Big)$ from $D_i$;
\STATE $s_i^{(r)} \gets \frac{\text{err}_{T_i}(f^{(r)})\eps'\alpha'}{(1+3\alpha')\text{err}_{S^{(r)}}(f^{(r)}) + 3\eps'}$
\STATE \textbf{Update}: $w_i^{(r)} \gets w_i^{(r-1)} (1 + s_i^{(r)})$
\ENDFOR
\ENDFOR
\STATE \RETURN $f_{\text{NR2-AVG}}$, where $f_{\text{NR2-AVG}}(x) \stackrel{R}{\gets} \{f^{(r)}(x)\}_{r=1}^t$;
\end{algorithmic}
\end{algorithm}

We first prove the guarantee for Algorithm NR1-AVG.

\begin{proof}[Proof of Theorem~\ref{thNR1-AVG}]
The expected error of the returned classifier $f_{\text{NR1-AVG}}$ on player $i$'s distribution is $\overline{\text{err}}_{D_i} (f_{\text{NR1-AVG}})= \frac{1}{t}\sum_{r=1}^t\text{err}_{D_i}(f^{(r)})$. We will prove that with probability at least $1-\delta$, $\overline{\text{err}}_{D_i} (f_{\text{NR1-AVG}})\leq (1+\alpha)\OPT+\eps$ for all $i\in [k]$.

By the Corollary, for a given round $r$ and player $i$,
\[\Pr[|\text{err}_{T_i}(f^{(r)}) -\text{err}_{D_i}(f^{(r)})| \ge \alpha' \cdot \text{err}_{D_i}(f^{(r)})+\eps'] \le 2\exp(-\alpha'\eps' |T_i| / 3).\]
If $\vert T_i \vert =\frac{3}{\eps'\alpha'}\ln\Big(\frac{k}{\delta'}\Big) = O\Big(\frac{1}{\eps'\alpha'}\ln\Big(\frac{k}{\delta'}\Big)\Big)$, the inequality 
holds with probability at least $1-2\delta'/k$. By union bound, it follows that it holds for every $i$ and every $r$ with probability at least $1-2\delta't=1-\delta/2$.

With probability at least $1-\delta$ the previous inequality as well as the inequality of Lemma~\ref{empirical_sum} hold for all rounds and players. We restrict the rest of the proof to this event.

It holds that,
\begin{align*}
\Phi^{(r)} & = \Phi^{(r-1)} +\sum_{i=1}^k \left(w_i^{(r-1)} s_i^{(r)}\right) \\
&\le \Phi^{(r-1)} + \frac{\eps'\alpha'}{(1+3\alpha')\text{err}_{S^{(r)}}(f^{(r)}) + 3\eps'}\sum_{i=1}^k \left(w_i^{(r-1)} \text{err}_{T_i}(f^{(r)})\right) \\
&\stackrel{L_r=[k]}{\leq} \Phi^{(r-1)}\left(1+\frac{\eps'\alpha'}{(1+3\alpha')\text{err}_{S^{(r)}}(f^{(r)}) + 3\eps'}[(1+3\alpha')\text{err}_{S^{(r)}}(f^{(r)}) + 3\eps']\right)\\
&\le \Phi^{(r-1)}(1+\eps'\alpha')
\end{align*}

By induction, $\Phi^{(t)}\le k\exp(t\eps'\alpha')$.
From Lemma~\ref{s_sum} and since $t = 2\lceil\ln (k) / (\eps'\alpha'^2) \rceil$, it follows that 
\begin{equation}\label{s_sum5}
\sum_{r=1}^t s_i^{(r)} \le \frac{\ln (k) + t\eps'\alpha'}{1-\alpha'/2} \le \frac{1+\alpha'}{1-\alpha'/2}t\eps'\alpha'.
\end{equation}

Therefore, the total error is:
\begin{align*}
\sum_{r=1}^t \text{err}_{D_i}(f^{(r)}) 
&\leq \sum_{r=1}^t \frac{\text{err}_{T_i}(f^{(r)})}{1-\alpha'} + \frac{t\eps'}{1-\alpha'}\\
&\leq \sum_{r=1}^t \frac{(1+3\alpha')\text{err}_{S^{(r)}}(f^{(r)}) + 3\eps'}{\eps'\alpha'} \frac{\text{err}_{T_i}(f^{(r)}) \eps'\alpha'}{(1-\alpha')[(1+3\alpha')\text{err}_{S^{(r)}}(f^{(r)}) + 3\eps']} + \frac{t\eps'}{1-\alpha'}\\
&=\sum_{r=1}^t  \frac{(1+3\alpha')\text{err}_{S^{(r)}}(f^{(r)}) + 3\eps'}{(1-\alpha')\eps'\alpha'} s_i^{(r)} + \frac{t\eps'}{1-\alpha'}\\
& \stackrel{(\ref{s_sum5})}{\leq} \frac{(1+7\alpha')\OPT+19\eps'}{(1-\alpha')\eps'\alpha'}\frac{(1+\alpha')}{1-\alpha'/2}t\eps'\alpha' + \frac{t\eps'}{1-\alpha'}\\
& \leq [(1+12\alpha')\OPT+25\eps']t\\
& = [(1+\alpha)\OPT+\eps]t,
\end{align*}
where the last inequality holds for $\alpha'<1/12$ and we have set $\alpha' = \alpha/12$ and $\eps' = \eps/25$.

As for the total number of samples, it is the sum of $O(\frac{k}{\alpha'\eps'}\ln(k/\delta'))$ samples and $O\Big(\frac{1}{\alpha'\eps'} \Big(d\ln \Big(\frac{1}{\eps'}\Big)+\ln \Big(\frac{1}{\delta'}\Big)\Big)\Big)$ samples for each round.
Because there are $O(\ln (k)/(\eps'\alpha'^2))$ rounds, the total number of samples is
\[O\Big(\frac{\ln(k)}{\alpha^3\eps^2}\Big(k\ln\left(\frac{k}{\delta}\right)+d\ln \Big(\frac{1}{\eps}\Big) \Big)\Big).\]
\end{proof}

Finally, we prove the guarantee of Algorithm NR2-AVG.

\begin{proof}[Proof of Theorem~\ref{thNR2-AVG}]
The expected error of the returned classifier $f_{\text{NR2-AVG}}$ on player $i$'s distribution is $\overline{\text{err}}_{D_i} (f_{\text{NR2-AVG}})= \frac{1}{t}\sum_{r=1}^t\text{err}_{D_i}(f^{(r)})$. We will prove that with probability at least $1-\delta$, $\overline{\text{err}}_{D_i} (f_{\text{NR2-AVG}})\leq (1+\alpha)\OPT+\eps$ for all $i\in [k]$.

By the Corollary, for a given round $r$ and player $i$,
\[\Pr[|\text{err}_{T_i}(f^{(r)}) -\text{err}_{D_i}(f^{(r)})| \ge \alpha' \cdot \text{err}_{D_i}(f^{(r)})+\eps'] \le 2\exp(-\alpha'\eps' |T_i| / 3).\]
If $\vert T_i \vert =\frac{3}{\eps'\alpha'}\ln\Big(\frac{2}{\eps'\alpha'}\Big) \stackrel{\alpha'>2\eps'}{=} O\Big(\frac{1}{\eps'\alpha'}\ln\Big(\frac{1}{\eps'}\Big)\Big)$, then
\begin{equation}\label{empirical6}
\Pr[|\text{err}_{T_i}(f^{(r)}) -\text{err}_{D_i}(f^{(r)})| \ge \alpha' \cdot \text{err}_{D_i}(f^{(r)})+\eps'] \le \eps'\alpha'.
\end{equation}

Assuming that the inequality of Lemma~\ref{empirical_sum} holds, which is true with probability $1-\delta/2$, it follows that
\begin{align*}
&\E[\Phi^{(r)} \mid \Phi^{(r-1)}]\\
&= \E\left[\Phi^{(r-1)} + \frac{\eps'\alpha'}{(1+3\alpha')\text{err}_{S^{(r)}}(f^{(r)}) + 3\eps'}\sum_{i\in L_r} \left(w_i^{(r-1)} \text{err}_{T_i}(f^{(r)})\right) + \sum_{i\notin L_r} \left(w_i^{(r-1)}s_i^{(r-1)}\right) \middle| \Phi^{(r-1)}\right]\\
& \le \E\left[\Phi^{(r-1)} + \frac{\eps'\alpha'}{(1+3\alpha')\text{err}_{S^{(r)}}(f^{(r)}) + 3\eps'}[(1+3\alpha')\text{err}_{S^{(r)}}(f^{(r)}) + 3\eps'] \Phi^{(r-1)}+ \alpha'\sum_{i\notin L_r} w_i^{(r-1)} \middle| \Phi^{(r-1)}\right]\\
& \stackrel{(\ref{empirical6})}{\le} \Phi^{(r-1)}(1+\eps'\alpha'+\eps'\alpha'^2)
\end{align*}

By the definition of expectation, $\E[\Phi^{(r)}] \leq  \E[\Phi^{(r-1)}](1+\eps'\alpha'+\eps'\alpha'^2)$. So by induction, $\E[\Phi^{(t)}]\le k\exp(t\eps'\alpha'(1+\alpha'))$. 
Markov's inequality states that $\Pr[\Phi^{(t)} \ge \frac{\E[\Phi^{(t)}]}{\delta/4}] \leq \delta/4$. So with probability $1-\delta/4-\delta/2 = 1-3\delta/4$ it holds that $\Phi^{(t)} \le \frac{4k}{\delta}\exp(t\eps'\alpha'(1+\alpha'))$.

From Lemma~\ref{s_sum} and $t=2\lceil\ln (4k/\delta) / (\eps'\alpha'^2) \rceil$, it follows that 
\begin{equation}\label{s_sum6}
\sum_{r=1}^t s_i^{(r)} \le \frac{\ln (4k/\delta) + t\eps'\alpha'(1+\alpha')}{1-\alpha'/2} \le \frac{(1+2\alpha')}{1-\alpha'/2}t\eps'\alpha'.
\end{equation}

Let $R_i=\{r\in[t] \mid |\text{err}_{T_i}(f^{(r)}) -\text{err}_{D_i}(f^{(r)})| \leq \alpha' \cdot \text{err}_{D_i}(f^{(r)})+\eps'\}$. For the classifiers of the rounds $r\in R_i$:
\begin{align*}
\sum_{r\in R_i} \text{err}_{D_i}(f^{(r)})
&\leq \sum_{r\in R_i}\frac{\text{err}_{T_i}(f^{(r)})}{1-\alpha'} + \frac{|R_i|\eps'}{1-\alpha'}\\
&\leq \sum_{r\in R_i} \frac{(1+3\alpha')\text{err}_{S^{(r)}}(f^{(r)}) + 3\eps'}{\eps'\alpha'} \frac{\text{err}_{T_i}(f^{(r)}) \eps'\alpha'}{(1-\alpha')[(1+3\alpha')\text{err}_{S^{(r)}}(f^{(r)}) + 3\eps']} + \frac{t\eps'}{1-\alpha'}\\
&= \sum_{r\in R_i}\frac{(1+3\alpha')\text{err}_{S^{(r)}}(f^{(r)}) + 3\eps'}{(1-\alpha')\eps'\alpha'} s_i^{(r)} + \frac{t\eps'}{1-\alpha'}\\
&\stackrel{(\ref{s_sum6})}{\leq} \frac{(1+7\alpha')\OPT+19\eps'}{(1-\alpha')\eps'\alpha'}\frac{(1+2\alpha')}{1-\alpha'/2}t\eps'\alpha' + \frac{t\eps'}{1-\alpha'}\\
& \leq [(1+15\alpha')\OPT+25\eps']t
\end{align*}
which holds for $\alpha'<1/15$.

We will now bound $|[t]\setminus R_i|$. For every round $r$, let $m^{(r)}$ be the indicator random variable of the set $[t]\setminus R_i$ and let $y^{(r)}=\eps'\alpha'$. It holds that for all rounds $r$, $|m^{(r)}-y^{(r)}| \le 1$ and $m^{(r)}, y^{(r)} \ge 0$. In addition, from inequality~(\ref{empirical6}) it follows that $\E[m^{(r)}-y^{(r)} \mid \sum_{r'<r}m^{(r')}, \sum_{r'<r}y^{(r')}] = \eps'\alpha'-\eps'\alpha' \leq 0$.

Using [\cite{KY14}, Lemma 10], with $\varepsilon=1/2$ and $A=\eps'\alpha'$, we get that \[\Pr\left[\sum_{r=1}^t m^{(r)} \geq 2\eps'\alpha't + 2\eps'\alpha't\right] \leq \exp(-\eps'\alpha't/2) \leq \delta/4k.\]

So $|[t]\setminus R_i| = \sum_{r=1}^t m^{(r)} \leq 4\eps'\alpha't$ for all $i$ with probability at least $1-\delta/4$.

Thus, for the expected error it holds that:
\begin{align*}
\frac{\sum\limits_{r=1}^t \text{err}_{D_i}(f^{(r)})}{t} 
&= \frac{\sum\limits_{r\in R_i} \text{err}_{D_i}(f^{(r)}) + \sum\limits_{r\notin R_i} \text{err}_{D_i}(f^{(r)})}{t}\\
&\leq (1+15\alpha')\OPT+25\eps'+ 4\eps'\alpha' \le (1+15\alpha')\OPT + 29\eps'.
\end{align*}
Having set $\alpha'=\alpha/15$ and $\eps'=\eps/29$ we get that $\overline{\text{err}}_{D_i}(f_{\text{NR2-AVG}}) \leq (1+\alpha)\OPT+\eps$ with probability at least $1-\delta$.

As for the total number of samples, it is the sum of $O(\frac{k}{\alpha'\eps'}\ln(1/\eps'))$ samples and $O\Big(\frac{1}{\alpha'\eps'} \Big(d\ln \Big(\frac{1}{\eps'}\Big)+\ln \Big(\frac{1}{\delta'}\Big)\Big)\Big)$ samples for each round.
Because there are $O(\ln(k/\delta) / \eps'\alpha'^2)$ rounds, the total number of samples is
\[O\Big(\frac{1}{\alpha^3\eps^2} \ln \Big(\frac{k}{\delta}\Big) \Big((d+k)\ln \Big(\frac{1}{\eps}\Big) + \ln\Big(\frac{1}{\delta}\Big)\Big)\Big).\]
\end{proof}

\section{Discussion}
The problem has four parameters, $d$, $k$, $\eps$ and $\delta$, so there are many ways to compare the sample complexity of the algorithms. In the non-realizable setting there is one more parameter $\alpha$, but this is set to be a constant in the beginning of the algorithms. Our sample complexity upper bounds are summarized in the following table.

\begin{table}[H]
  \caption{Sample complexity upper bounds}\label{table}
  \centering
  \begin{tabular}{p{0.15\textwidth}p{0.33\textwidth}p{0.45\textwidth}} \\
    \toprule
    & Algorithm 1     & Algorithm 2    \\
    \midrule
    Realizable & 
    $O\Big(\frac{\ln (k)}{\eps}\Big(d\ln\Big(\frac{1}{\eps}\Big)+k\ln\Big(\frac{k}{\delta}\Big)\Big)\Big)$ & 
   $O\Big(\frac{\ln(k/\delta)}{\eps}\Big(d\ln\Big(\frac{1}{\eps}\Big)+k+\ln \Big(\frac{1}{\delta}\Big)\Big) \Big)$  \\
    
    Non-realizable ($2+\alpha$ approx.) & 
    $O\Big(\frac{\ln (k)}{\alpha^4 \eps}\Big(d\ln\Big(\frac{1}{\eps}\Big)+k\ln\Big(\frac{k}{\delta}\Big)\Big)\Big)$     & 
    $O\Big(\frac{\ln(k/\delta)}{\alpha^4 \eps}\Big(d\ln\Big(\frac{1}{\eps}\Big)+k\ln\Big(\frac{1}{\alpha}\Big) + \ln\Big(\frac{1}{\delta}\Big)\Big)\Big)$ \\
    
    Non-realizable  (randomized)  &
     $O\Big(\frac{\ln (k)}{\alpha^3 \eps^2}\Big(d\ln\Big(\frac{1}{\eps}\Big)+k\ln\Big(\frac{k}{\delta}\Big)\Big)\Big)$       & 
     $O\Big(\frac{\ln(k/\delta)}{\alpha^3 \eps^2}\Big((d+k)\ln\Big(\frac{1}{\eps}\Big)+\ln\Big(\frac{1}{\delta}\Big)\Big)\Big)$ \\
    \bottomrule
  \end{tabular}
\end{table}

Usually $\delta$ can be considered constant, since it represents the required error probability, or, in the high success probability regime, $\delta=\frac{1}{poly(k)}$. For both of these natural settings, we can see that Algorithm 2 is better than Algorithm 1, except for the case of the expected error guarantee. If we assume $k=\Theta(d)$, then Algorithm 2 is always better than Algorithm 1. 

In the realizable setting, Algorithm R1 is always better than the algorithm of~\cite{BHPQ17} for the centralized variant of the problem and matches their number of samples in the personalized variant.
In addition, Theorem 4.1 of~\cite{BHPQ17} states that the sample complexity of any algorithm in the collaborative model is $\Omega\Big(\frac{k}{\eps} \ln \Big(\frac{k}{\delta}\Big)\Big)$, given that $d=\Theta(k)$ and $\eps,\delta \in (0,0.1)$, and this holds even for the personalized variant. For $d=\Theta(k)$, the sample complexity of Algorithm R2 is exactly $\ln \Big(\frac{k}{\delta}\Big)$ times the sample complexity for learning one task. Furthermore, when $\vert \mathcal{F} \vert = 2^d$ (e.g. the hard instance for the lower bound of~\cite{BHPQ17}), only $m_{\eps,\delta} = O\Big(\frac{1}{\eps}\Big (d + \ln \Big(\frac{1}{\delta}\Big)\Big)\Big)$ samples are required in the non-collaborative setting instead of the general bound of the VC theorem, so the sample complexity bound for Algorithm R2 is $O\Big(\ln \Big(\frac{k}{\delta}\Big)\frac{1}{\eps}\Big (d + k+\ln \Big(\frac{1}{\delta}\Big)\Big)\Big)$ and matches exactly the lower bound of~\cite{BHPQ17} up to lower order terms.

In the non-realizable setting, our generalization of algorithms R1 and R2, NR1 and NR2 respectively, have the same sample complexity as in the realizable setting and match the error guarantee for $\OPT=0$. If $\OPT \neq 0$, they guarantee an error of a factor $2$ multiplicatively on $\OPT$. The randomized classifiers returned by Algorithms NR1-AVG and NR2-AVG avoid this factor of $2$ in their expected error guarantee. However, to learn such classifiers, there are required $O\Big(\frac{1}{\eps}\Big)$ times more samples.

\nocite{*}
\bibliography{CollPAC_NIPS}

\begin{thebibliography}{14}
\providecommand{\natexlab}[1]{#1}
\providecommand{\url}[1]{\texttt{#1}}
\expandafter\ifx\csname urlstyle\endcsname\relax
  \providecommand{\doi}[1]{doi: #1}\else
  \providecommand{\doi}{doi: \begingroup \urlstyle{rm}\Url}\fi

\bibitem[Anthony and Bartlett(2009)]{AB09}
Martin Anthony and Peter~L. Bartlett.
\newblock \emph{Neural Network Learning: Theoretical Foundations}.
\newblock Cambridge University Press, New York, NY, USA, 1st edition, 2009.
\newblock ISBN 052111862X, 9780521118620.

\bibitem[Balcan et~al.(2012)Balcan, Blum, Fine, and Mansour]{BBFM12}
Maria{-}Florina Balcan, Avrim Blum, Shai Fine, and Yishay Mansour.
\newblock Distributed learning, communication complexity and privacy.
\newblock In \emph{Proceedings of the 25th Conference on Computational Learning
  Theory (COLT)}, pages 26.1--26.22, 2012.

\bibitem[Baxter(1997)]{Baxter97}
Jonathan Baxter.
\newblock A {B}ayesian/information theoretic model of learning to learn via
  multiple task sampling.
\newblock \emph{Machine Learning}, 28\penalty0 (1):\penalty0 7--39, July 1997.
\newblock ISSN 0885-6125.
\newblock \doi{10.1023/A:1007327622663}.

\bibitem[Baxter(2000)]{Baxter00}
Jonathan Baxter.
\newblock A model of inductive bias learning.
\newblock \emph{Journal of Artificial Intelligence Research}, 12\penalty0
  (1):\penalty0 149--198, March 2000.
\newblock ISSN 1076-9757.

\bibitem[Blum et~al.(2017)Blum, Haghtalab, Procaccia, and Qiao]{BHPQ17}
Avrim Blum, Nika Haghtalab, Ariel~D. Procaccia, and Mingda Qiao.
\newblock Collaborative {PAC} learning.
\newblock In \emph{Proceedings of the 30th Annual Conference on Neural
  Information Processing Systems (NIPS)}, pages 2389--2398, 2017.

\bibitem[Chen et~al.(2018)Chen, Zhang, and Zhou]{CZZ18}
Jiecao Chen, Qin Zhang, and Yuan Zhou.
\newblock Tight bounds for collaborative {PAC} learning via multiplicative
  weights, 2018.

\bibitem[Dekel et~al.(2011)Dekel, Gilad-Bachrach, Shamir, and Xiao]{DGSX11}
Ofer Dekel, Ran Gilad-Bachrach, Ohad Shamir, and Lin Xiao.
\newblock Optimal distributed online prediction.
\newblock In \emph{Proceedings of the 28th International Conference on Machine
  Learning (ICML)}, pages 713--720, 2011.

\bibitem[Finn et~al.(2017)Finn, Abbeel, and Levine]{FAL17}
Chelsea Finn, Pieter Abbeel, and Sergey Levine.
\newblock Model-agnostic meta-learning for fast adaptation of deep networks.
\newblock In \emph{Proceedings of the 34th International Conference on Machine
  Learning (ICML)}, pages 1126--1135, 2017.

\bibitem[Koufogiannakis and Young(2014)]{KY14}
Christos Koufogiannakis and Neal~E. Young.
\newblock A nearly linear-time {PTAS} for explicit fractional packing and
  covering linear programs.
\newblock \emph{Algorithmica}, 70\penalty0 (4):\penalty0 648--674, December
  2014.
\newblock ISSN 0178-4617.
\newblock \doi{10.1007/s00453-013-9771-6}.

\bibitem[Mansour et~al.(2009{\natexlab{a}})Mansour, Mohri, and
  Rostamizadeh]{MMR09-1}
Yishay Mansour, Mehryar Mohri, and Afshin Rostamizadeh.
\newblock Domain adaptation: Learning bounds and algorithms.
\newblock In \emph{Proceedings of the 22nd Conference on Computational Learning
  Theory (COLT)}, pages 19--30, 2009{\natexlab{a}}.

\bibitem[Mansour et~al.(2009{\natexlab{b}})Mansour, Mohri, and
  Rostamizadeh]{MMR09-2}
Yishay Mansour, Mehryar Mohri, and Afshin Rostamizadeh.
\newblock Domain adaptation with multiple sources.
\newblock In \emph{Proceedings of the 23rd Annual Conference on Neural
  Information Processing Systems (NIPS)}, pages 1041--1048, 2009{\natexlab{b}}.

\bibitem[Mitzenmacher and Upfal(2017)]{MU17}
Michael Mitzenmacher and Eli Upfal.
\newblock \emph{Probability and Computing: Randomization and Probabilistic
  Techniques in Algorithms and Data Analysis}.
\newblock Cambridge University Press, 2nd edition, 2017.
\newblock ISBN 110715488X, 9781107154889.

\bibitem[Valiant(1984)]{Valiant84}
L.~G. Valiant.
\newblock A theory of the learnable.
\newblock \emph{Commun. ACM}, 27\penalty0 (11):\penalty0 1134--1142, November
  1984.
\newblock ISSN 0001-0782.
\newblock \doi{10.1145/1968.1972}.

\bibitem[Wang et~al.(2016)Wang, Kolar, and Srebro]{WKS16}
Jialei Wang, Mladen Kolar, and Nathan Srebro.
\newblock Distributed multi-task learning.
\newblock In \emph{Proceedings of the 19th International Conference on
  Artificial Intelligence and Statistics (AISTATS)}, pages 751--760, 2016.

\end{thebibliography}
\bibliographystyle{plain}

\end{document}